\documentclass{article}

\usepackage{microtype}
\usepackage{graphicx}
\usepackage{subfigure}
\usepackage{booktabs} 
\usepackage[accepted]{icml2024}
\usepackage{amsmath,amsfonts,amsthm,amssymb,multirow,mathtools}
\usepackage{multicol}
\usepackage{algorithmic,comment,url,enumerate}
\usepackage{tikz}
\usepackage{framed}
\usepackage{bbm}
\usepackage{bm}
\usepackage{commath}
\usetikzlibrary{decorations.pathreplacing, shapes}
\usepackage{xspace}

\usepackage[backref, colorlinks,citecolor=blue,linkcolor=magenta,bookmarks=true]{hyperref} 
\usepackage{dblfloatfix}    
\usepackage{placeins}

\usepackage[nameinlink,capitalize,noabbrev]{cleveref}

\usepackage{color-edits}

\usepackage{thm-restate}

\addauthor{ev}{purple}
\addauthor{aw}{red}
\addauthor{ah}{green}
\addauthor{as}{blue}

\definecolor{shadecolor}{gray}{0.95}

\newcommand{\mcN}{\mathcal{N}}

\newcommand{\CR}{\operatorname{CR}}

\renewcommand \vec [1]{\bm{#1}}

\renewcommand{\phi}{\varphi}
\renewcommand{\epsilon}{\varepsilon}
\newcommand{\E}{\mathbb{E}}

\newcommand{\eps}{\varepsilon}
\definecolor{shadecolor}{gray}{0.95}

\newcommand{\R}{\mathbb{R}}

\newcommand{\Z}{\mathbb{Z}}

\newcommand{\cG}{\mathcal{G}}
\newcommand{\cV}{\mathcal{V}}

\newcommand{\cD}{\mathcal{D}}

\newcommand{\RH}{\text{OBBM}\xspace}
\newcommand{\onOPT}{\textsc{$\text{OPT}_{on}$}\xspace}
\newcommand{\OPT}{\textsc{OPT}\xspace}
\newcommand{\MG}{\textsc{Magnolia}\xspace}

\newcommand{\greedy}{\texttt{greedy}\xspace}
\newcommand{\greedyt}{\texttt{greedy-t}\xspace}
\newcommand{\LP}{\texttt{LP-rounding}\xspace}
\newcommand{\x}{\texttimes}

\theoremstyle{plain}
\newtheorem{theorem}{Theorem}[section]

\newtheorem{definition}[theorem]{Definition}

\theoremstyle{definition}

\theoremstyle{remark}

\newif\ifshow 
\showtrue 

\ifshow
  \includecomment{proofwrap}
\else
  \excludecomment{proofwrap}
\fi

\icmltitlerunning{MAGNOLIA}

\begin{document}

\twocolumn[
\icmltitle{\textbf{MAGNOLIA}:\\ \textbf{M}atching \textbf{A}lgorithms via \textbf{GN}Ns for \textbf{O}n\textbf{li}ne 
Value-to-go \textbf{A}pproximation}



\icmlsetsymbol{equal}{*}

\begin{icmlauthorlist}
\icmlauthor{Alexandre Hayderi}{aaa}
\icmlauthor{Amin Saberi}{bbb}
\icmlauthor{Ellen Vitercik}{aaa,bbb}
\icmlauthor{Anders Wikum}{bbb}
\end{icmlauthorlist}

\icmlaffiliation{aaa}{Department of Computer Science, Stanford University, Stanford, CA, USA}
\icmlaffiliation{bbb}{Department of Management Science \& Engineering, Stanford University, Stanford, CA, USA}

\icmlcorrespondingauthor{Anders Wikum}{wikum@stanford.edu}

\icmlkeywords{Matching, Online Algorithms, Rideshare, Combinatorial Optimization, Graph Neural Networks}

\vskip 0.3in
]

\printAffiliationsAndNotice{} 

\begin{abstract}
    Online Bayesian bipartite matching is a central problem in digital marketplaces and exchanges, including advertising, crowdsourcing, ridesharing, and kidney exchange. We introduce a graph neural network (GNN) approach that emulates the problem's combinatorially-complex optimal online algorithm, which selects actions (e.g., which nodes to match) by computing each action's \emph{value-to-go (VTG)}---the expected weight of the final matching if the algorithm takes that action, then acts optimally in the future. We train a GNN to estimate VTG and show empirically that this GNN returns high-weight matchings across a variety of tasks. Moreover, we identify a common family of graph distributions in spatial crowdsourcing applications, such as rideshare, under which VTG can be efficiently approximated by aggregating information within local neighborhoods in the graphs. This structure matches the local behavior of GNNs, providing theoretical justification for our approach.
\end{abstract}

\section{Introduction}

Online matching is a critical problem in many digital marketplaces. 
In advertising, website visitors are matched to ads~\citep{Mehta13:Online}, and on crowdsourcing platforms, crowdworkers are matched to appropriate tasks~\citep{Tong20:Spatial}. The rideshare industry faces the complex task of matching riders with drivers~\citep{Zhao19:Oreference}, while in medical fields such as kidney exchange, donors must be efficiently matched to patients~\citep{Ezra20:Online}. 
The challenge is that irrevocable matching decisions must be made online without precise knowledge of how demand will evolve.

In the notoriously challenging \emph{Online Bayesian Bipartite Matching (OBBM)} problem, there is a bipartite graph with \emph{online} and \emph{offline} nodes and weighted edges. The node sets could, for example, represent crowdworkers and tasks, with weights encoding workers' payoffs for completing tasks.
Matching occurs over a series of rounds, with one online node arriving with known probability in each round. Upon a successful arrival, a matching algorithm must determine whether to match the node to an offline node. Alternatively, the algorithm can skip the node, in which case it will never be matched, leaving nodes available for the future. The goal is to compute a high-weight matching. However, it is not known which online nodes will arrive \emph{a priori}.

The online optimal algorithm for \RH, \onOPT, knows the distribution over online node arrivals but not the future realization of arriving nodes. Upon the arrival of an online node, \onOPT takes the action (i.e., the choice of the matching edge or the decision to skip) that maximizes the weight of the final matching in expectation over the future node arrivals. 
In more detail, \onOPT can be formulated as a dynamic programming (DP) routine. At each timestep, it computes the \emph{value-to-go (VTG)} of each action, which is the expected final matching weight if it (1) takes that action and then (2) takes each subsequent action to maximize the expected weight of the final matching. The DP routine takes the action with the highest VTG.

\subsection{Our contributions}
We train a graph neural network (GNN) that empirically competes with this omnipotent optimal algorithm by estimating the VTG of each action. Moreover, we provide a theoretical analysis that helps justify this deep learning architecture's suitability for the OBBM problem.

\paragraph{Key challenges.} The primary obstacle we face is the sheer complexity of \RH: for some constant $\alpha < 1$, it is PSPACE-hard to approximate the value of \onOPT within a factor of $\alpha$ ~\cite{Papadimitriou21:Online}. Moreover, unless PSPACE = BPP, 
no polynomial-time algorithm can replicate the decisions of an online algorithm that $\frac{\alpha + 1}{2}$-approximates \onOPT.
Nonetheless, we show empirically that GNNs compete with \onOPT across many tasks.
Moreover, although deep learning architectures are notoriously difficult to analyze theoretically, we prove a correspondence between the functions computable by GNNs and the VTG function. We prove that for a broad family of graph distributions, VTG can be efficiently approximated by aggregating information within local neighborhoods in the graphs.
This structure matches the behavior of GNNs: over a series of rounds, each node computes a \emph{message}, which is a function of its \emph{internal representation} (a feature vector), passes this message to its neighbors, and updates its internal representation using its received messages. Thus, GNNs are ideal for approximating functions that only depend on a graph's local neighborhood. We summarize our two primary contributions as follows.
\paragraph{Theoretical guarantees.}
We prove that for \emph{bipartite random geometric graphs (b-RGGs)}, VTG can be efficiently approximated by only aggregating information within small local neighborhoods. In a b-RGG, edges are formed based on the similarity of node embeddings within a latent space. Thus, b-RGGs often arise in spatial crowdsourcing~\citep{Tong20:Spatial} such as ridesharing, which necessitates physical proximity between drivers and riders. We prove that b-RGGs can be decomposed into smaller, local subgraphs with limited interconnectivity. This structure allows us to prove that VTG can be estimated by a function that only aggregates information within these local subgraphs.

\paragraph{Empirical analysis.} We present \MG~(\textbf{M}atching \textbf{A}lgorithms via \textbf{GN}Ns for \textbf{O}n\textbf{li}ne Value-to-go \textbf{A}pproximation), a GNN-based online matching framework  that mimics the actions of $\onOPT$ by predicting the VTG of each feasible matching decision. While computing VTG is intractable for even moderately sized graphs, we show empirically that a supervised learning approach is still effective; despite being trained on graphs with 16 total nodes, \MG beats state-of-the-art baseline algorithms across a broad range of problem types, sizes, and regimes, showing strong out-of-distribution generalization.

\subsection{Related work}
\paragraph{Online Bayesian bipartite matching.} 
\RH is connected to the literature on online Bayesian selection problems, where a seminal result by \citet{Krengel78:Semiarts} implies that no algorithm can provide better than a $0.5$-approximation to the offline optimal matching in hindsight. However, better approximation ratios are possible when competing with the optimal \emph{online} algorithm \onOPT. \citet{Papadimitriou21:Online} gave a $0.51$-approximation algorithm for \onOPT, a bound subsequently improved by \citet{Braverman22:Max-weight} and \citet{Naor23:Online}. Our experiments demonstrate that our approach consistently outperforms these approximation algorithms.

\paragraph{ML for combinatorial optimization.} There has been significant recent interest in integrating ML with combinatorial optimization (see, e.g., the surveys by \citet{Bengio21:Machine} and \citet{Cappart23:Combinatorial}).
Applications of ML to online NP-hard problems have primarily aimed to learn algorithms with good worst-case guarantees \citep[e.g.,][]{Kong18:New,Zuzic20:Learning,Du22:Adversarial}. The work by \citet{Alomrani22:Deep} is one of the few that studies average-case performance. They present an end-to-end RL framework for learning online bipartite matching policies in the unknown i.i.d. arrival setting using GNNs. Their approach differs from ours in a few critical ways. (i) Whereas the structure of the optimal online algorithm is not known in the unknown i.i.d. setting, in \RH, the optimal online algorithm is simple to express but computationally intractable, resulting in a fundamentally different ML task. (ii) The existence of good approximation algorithms for \RH allows us to compare \MG's performance to stronger benchmarks than those available in the unknown i.i.d. setting. (iii) \citeauthor{Alomrani22:Deep}'s paper is empirical, analyzing the performance of various models to identify which underlying characteristics make them perform well. In contrast, we provide theoretical justification for a GNN's ability to replicate the decisions of an optimal online algorithm on real-world graphs, as well as experiments. \citet{Li23:Learning} bridge the gap between worst-case guarantees and average-case performance for online matching by switching between expert and ML predictions online. Using VTG predictions from \MG  within this switching framework yields an algorithm which is competitive against any fixed online algorithm.

\section{Notation and Background}
    Let $G=(L,R, E)$ be a bipartite graph on a set $L$ of $n$ \emph{offline nodes} and a set $R = \{1, \dots, m\}$ of $m$ \emph{online nodes}, with undirected edges $E \subseteq L \times R$. When the underlying graph $G$ is not clear from context, we refer to these sets as $L(G)$, $R(G)$, and $E(G)$. We use the notation $\mathcal{N}_G(t)$ to denote the neighbors of online node $t \in R$ in $G$, and $N=m+n$ to denote the total number of nodes.

\subsection{Online Bayesian bipartite matching (OBBM)}
\label{subsection:obm}
An input to the \RH problem is a bipartite graph $G=(L,R,E)$ attributed with edge weights $\{w_{ij}\}_{(i,j) \in E}$ and online node arrival probabilities $\{p_i\}_{i \in R}$. Matching in $G$ occurs over $m$ timesteps. At time $t \in [m]$, online node $t$ appears independently with probability $p_t.$ If node $t$ appears, one must irrevocably decide to match $t$ with an unmatched offline neighbor or skip $t$ and not match it to any node. The goal is to maximize the total weight of the final matching.

Although an algorithm for \RH knows the input graph $G$ from the onset, it does not know \emph{a priori} which online nodes will arrive.
Thus, in timestep $t$, if online node $t$ arrives and $S \subseteq L$ is the set of offline nodes that have not yet been matched, the algorithm's choice of which node to match $t$ to---or whether to skip $t$---is a function of $S$, $t$, and $G$.

The \emph{optimal online algorithm} \onOPT makes decisions that maximize the expected weight of the matching it returns over the randomness of the online node arrivals. In detail, \onOPT computes the Bellman equation---or \emph{value-to-go function}---$\cV_{G}(S, t)$, which is the expected value of the maximum weight matching achievable on $G$ over sequential arrivals $\{t, \dots, m\}$, with matchings restricted to the set of remaining offline nodes $S$.
Additional matches are only possible when there are available offline and online nodes, so 
$\cV_{G}(\emptyset, t) = 0$ for all $t \in R$ and $\cV_{G}(S, m+1) = 0$ for all $S \subseteq L$. The values of $\cV_G(\cdot)$ are related by the recurrence 
\begin{align*}
\cV_G(S&, t) = 
(1 - \vec{p}_t) \cdot \cV_G(S, t+1)\\
&+ \vec{p}_t \cdot \max\left\{\cV_G(S, t+1), \max_{u \in \mathcal{N}_G(t)\cap S} \cV_G(S,t,u)\right\}\end{align*}
where $\cV_G(S,t, u) = w_{tu} + \cV_G(S \setminus \{u\}, t+1)$ is the utility of matching $t$ to $u$: the edge $(t,u)$ is added to the matching, and $u$ becomes unavailable. If node $t$ does not arrive, the maximum expected matching value achievable over online nodes $\{t, \dots, m\}$ and offline node set $S$ is $\cV_G(S, t+1)$. If node $t$ does arrive, \onOPT matches $t$ to the best available offline node if $\max_{u \in \mathcal{N}_G(t)\cap S} \cV_G(S,t,u) > \cV_G(S, t + 1)$ and skips $t$ otherwise. We include pseudo-code for computing $\cV_G(S, t)$ in \cref{alg: vtg-dp}.

Throughout the paper, $\cV(G):= \cV_G(L(G), 1)$ refers to the full value-to-go computation on the input graph $G$, i.e., the expected value of the matching returned by \onOPT when run on $G$. If $H$ is the node-induced subgraph of $G$ over the vertex set $S \cup \{t, \dots, |R|\}$, then $\cV(H) = \cV_G(S, t)$. This observation will help to simplify some of the analysis; rather than proving statements over all $S \subseteq L$ and $t \in R$ for some larger graph $G$, we will instead prove statements about $\cV(H)$ over all attributed graphs $H$.

\subsection{Graph neural networks}
\label{subsection:gnns_background}
Let $G = (V, E, \mathcal{X})$ be a graph with node attributes $\mathcal{X} \subseteq \R^d$, so each node $v \in V$ is associated with an initial embedding $\vec{h}^{(0)}_v \in \mathcal{X}$.
A \emph{Message-Passing Graph Neural Network} (MPNN or GNN) of depth $k$ iteratively computes a sequence of embeddings $\vec{h}^{(1)}_v, \dots, \vec{h}^{(k)}_v$ for each node $v \in V$ starting from $\vec{h}_v^{(0)}$. 
In layer $i$, the GNN first computes a message $\vec{m}_v^{(i)}$ for each node $v$ from its previous embedding $\vec{h}_v^{(i-1)}$. Then, the next embedding $\vec{h}_v^{(i)}$ of node $v$ is computed by aggregating the messages $\vec{m}_u^{(i)}$ from each of $v$'s neighbors:
\begin{align*}\vec{m}_v^{(i)} &= \textsc{MSG}^{(i)}\left(\vec{h}_v^{(i-1)}\right) \qquad \text{for all }v \in V\\
\vec{h}_v^{(i)} &= \textsc{AGGREGATE}^{(i)}\left(\vec{m}_v^{(i)}, \{\vec{m}_u^{(i)} : u \in \mcN(v)\}\right).\end{align*}

These functions can contain learnable parameters, and different choices of these functions lead to different named GNN architectures. This framework also extends to graphs with edge attributes. By design, a single GNN can be trained and evaluated on graphs with any number of nodes.

A key observation is that the embedding $\vec{h}_v^{(k)}$ of a node $v$ is only a function of the embeddings within a $k$-hop neighborhood of $v$. In this way, GNNs are well-suited to learn functions that depend only on local substructures.

\section{Theoretical Guarantees}\label{sec:theory}
In this section, we identify conditions on the generating parameters of \emph{bipartite random geometric graphs} (b-RGGs) for which VTG can be approximated by aggregating information over local neighborhoods. b-RGGs are a random graph family that mimics the structure of real-world networks by generating edges according to the similarity between node embeddings in some latent space. As such, they are often used in \emph{spatial crowdsourcing} applications such as rideshare, where riders must be matched to nearby drivers ~\citep{Tong20:Spatial}.

In \cref{sec:graph_decomp}, we prove that b-RGGs can partitioned into small subgraphs with few edges crossing between subgraphs. Next, in \cref{section: local-approx}, we show that this property implies that VTG is locally approximable. This result aligns with the inherent local processing capabilities of GNNs, offering a theoretical foundation for our methodology. The full proofs of all results are in \cref{appendix: proofs}.

\subsection{Local graph decomposition}\label{sec:graph_decomp}
We begin by showing that, under certain conditions, b-RGGs admit a \emph{local decomposition}. Informally, this means that:

\begin{enumerate}
    \item[1.] (Decomposable) b-RGGs can be partitioned into subgraphs such that few edges cross between subgraphs.
    \item[2.] (Local) \label{property: balance} Under mild assumptions, the number of nodes in each resulting subgraph is relatively small.
\end{enumerate}
These properties are made formal in \cref{theorem: local-decomp}. Achieving both is nontrivial: a fine-grained partition may lead to small subgraphs, but it will likely result in many edges crossing between subgraphs.

\paragraph{Bipartite random geometric graphs.}\label{sbsection: b-RGGs} The defining characteristic of b-RGGs is that online and offline nodes are connected only when their embeddings are sufficiently close according to some metric. We prove results for $\ell_\infty$, though they immediately generalize to any $p$-norm.

\begin{definition}\label{def: b-RGG} Given a distribution $\cD$ over $[0,1]^d$, a \emph{bipartite random geometric graph} $G(m, n, \cD, \Delta)$ is a distribution over graphs on $m$ online and $n$ offline nodes where each node has an embedding $\vec{x_i} \sim \cD$. There is an edge between online node $i$ and offline node $j$ if and only if $\lVert\vec{x_i} - \vec{x_j}\rVert_\infty \leq \Delta$.
\end{definition}

A partition $\pi$ of $[0,1]^d$ induces a partition of b-RGGs into subgraphs: for $G \in \textup{supp}(G(m,n,\cD,\Delta))$, let $G(\pi)$ be the graph obtained from $G$ by removing all edges $(i,j)$ with embeddings $\vec{x_i}$ and $\vec{x_j}$ in different cells of $\pi$. We can thus map properties of the partition $\pi$ to properties of $G(\pi)$.

\paragraph{Random $k$-partitions.}
We introduce a random partitioning scheme that splits $[0,1]^d$ into cells of equal volume and applies a random ``shift" to these cells in each dimension. This shift is taken modulo 1, so the final cells are of equal volume but not necessarily contiguous.

\begin{definition}
\label{def: k-shift}
For $k \in \Z_{\geq 1}$, a \emph{$(k, \vec{s})$-partition} of $[0,1]^d$ is the partition $\vec{s}_i + \{0, \frac{1}{k}, \dots, \frac{k-1}{k}\}$ along each dimension $i$, where $\vec{s}_i \in [0, \frac{1}{k}]$.
\end{definition}

We use the notation $\Pi_k$ to denote the uniform distribution over $(k,\vec{s})$-partitions with $\vec{s} \sim \text{Unif}(0, 1/k)^d$ and refer to $\pi \sim \Pi_k$ as a \emph{random $k$-partition}. 

\paragraph{b-RGG decomposition.} For carefully chosen $k$, nearby vectors are likely to lie in the same cell of a random $k$-partition $\pi$. Since the edges of any b-RGG $G$ are based on proximity, this means each edge of $G$ is unlikely to be removed when forming $G(\pi)$.

\begin{restatable}{lemma}{Decomposable} 
\label{lemma: cut_prob}
Let $\vec{x_1}, \dots, \vec{x}_N \in  [0,1]^d$, $\epsilon > 0$, $\Delta \leq \frac{\epsilon}{2d}$, and $k = \lceil\frac{\epsilon}{2d\Delta}\rceil$. If $\lVert\vec{x_i} - \vec{x_j}\rVert_\infty \leq \Delta$, then with probability at most $\epsilon$ over $\pi \sim \Pi_k$, $\vec{x_i}$ and $\vec{x_j}$ lie in different cells of $\pi$.
\end{restatable}
\begin{proof}[Proof sketch]
    Notice that $\vec{x_i}$ and $\vec{x_j}$ lie in different cells of $\pi \sim \Pi_k$ precisely when, in at least one dimension $\ell$, some point of $\vec{s}_\ell + \{0, \frac{1}{k}, \dots, \frac{k-1}{k}\}$ lies in the interval $[(\vec{x_i})_\ell, (\vec{x_j})_\ell]$. By symmetry, this occurs in dimension $\ell$ independently with probability equal to the length of the interval $[(\vec{x_i})_\ell, (\vec{x_j})_\ell] \leq \Delta$ over the measure $1/k$ of possible shifts $\vec{s}_\ell$. For $k = \lceil\frac{\epsilon}{2d\Delta}\rceil$, the probability that this occurs in any dimension is at most $1 - (1 - k\Delta)^d \leq \epsilon$.
\end{proof}

Under a mild anti-concentration assumption on $\cD$, the number of b-RGG latent embeddings in any cell of a random $k$-partition with $\Omega(N)$ cells is likely sublinear in $N$. This does not hold, for example, when $\cD$ is a point mass. We avoid pathological cases with the concept of a $\beta$-smooth distribution \cite{Haghtalab22:Smoothed} from smoothed analysis.

\begin{definition}
    \label{def: beta-smooth}
    A distribution  $\cD$ over $[0,1]^d$ with probability density function $f$ is \emph{$\beta$-smooth} if $\sup f(\vec{x}) \leq \beta$.
\end{definition}
Our argument is based on a connection to balls-into-bins processes. A classic result guarantees that if $N$ balls are dropped uniformly at random into $N$ bins, the maximum load is $O(\ln N)$ with high probability~\cite{Mitzenmacher05:Probability}. \cref{lemma: balls-into-bins} slightly modifies this result.

\begin{restatable}{lemma}{ballsInBins}
\label{lemma: balls-into-bins}
For $\beta \geq 1$, when $N$ balls are dropped independently into $K=\Omega(N)$ bins and the probability a particular ball lands in each bin is at most $\frac{\beta}{K}$, the probability the maximum load is more than $\frac{3 \beta \ln N}{\ln \ln N}$ is $O(\frac{1}{N})$ for $N$ sufficiently large.
\end{restatable}

We treat sampling $N$ vectors from $\cD$ as a balls-into-bins process: balls are vectors, and bins are cells of a $k$-partition.

\begin{restatable}{corollary}{MaxLoad}\label{cor: max-load}
     Suppose $N$ vectors are sampled from a $\beta$-smooth distribution over $[0,1]^d$. For all $\pi \in \textup{supp}(\Pi_k)$ where $k^d = \Omega(N)$, every cell of $\pi$ contains $O(\beta \log N)$ vectors with probability 
    $1 - O(\frac{1}{N})$ for $N$ sufficiently large.
\end{restatable}
We now state this section's main theorem.
\begin{restatable}{theorem}{LocalDecomposition}
\label{theorem: local-decomp}
    Let $\cD$ be a $\beta$-smooth distribution, $\Delta = O(N^{-1/d})$, $\epsilon > 0$, and $k = \lceil\frac{\epsilon}{2d\Delta}\rceil.$ Then,
\begin{enumerate}
    \item (Decomposable) For any $G \in \textup{supp}(G(m, n, \cD, \Delta))$, each edge $e \in E(G)$ appears in $G(\pi)$ with probability at least $1 - \epsilon$ over the draw of $\pi \sim \Pi_k$.
    \item (Local) For any $\pi \in \textup{supp}(\Pi_k)$ and $N$ sufficiently large, the connected components of $G(\pi)$ are of size $O(\beta \log N)$ with probability $1 - O(1/N)$ over the draw of $G \sim G(m, n, \cD, \Delta)$.
\end{enumerate}
\end{restatable}
\begin{proof}[Proof sketch]
(1) follows from \cref{lemma: cut_prob} and (2) follows from  \cref{cor: max-load}.
\end{proof}

\subsection{Local approximation of value-to-go} \label{section: local-approx}
This section shows that the local decomposability of b-RGGs from \cref{theorem: local-decomp} means VTG can be approximated using \emph{local graph functions}. At a high level, these are functions that can be computed using only information from small neighborhoods.

\begin{definition}[\citet{Tahmasebi23:Power}]
\label{def:local}
    A function $f$ over graphs is \emph{$r$-local} if there is a function $\phi$ such that
    $f(G) = \phi(\{\mathcal{N}_r(v)\}_{v \in V(G)})$
where $\mathcal{N}_r(v)$ is the $r$-hop neighborhood of node $v$ in $G$.
\end{definition}

We prove that with high probability, VTG is approximated by a $O(\beta \log N)$-local function, formalized as follows.

\begin{definition}
    A function $f$ on graphs is \emph{$(r, \epsilon, \delta)$-locally approximable over a random graph family $\cG$} if there is an $r$-local, polynomial-time computable function $h$ such that
    $|f(G) - h(G)| \leq \epsilon f(G)$
    with probability $1-\delta$ over $G\sim \cG$ and any randomness in $h$.
\end{definition}

\paragraph{An initial (non-local) approximation.}
First, we prove that the VTG of an \RH instance can only decrease after removing the edges needed to form $G(\pi)$.

\begin{restatable}{lemma}{VtgUB}\label{lemma: vtg-upper-bound}
    For any $G \in \textup{supp}(G(m, n, \cD, \Delta)$ and any hypercube partition $\pi$,
$\cV(G(\pi)) \leq \cV(G).$
\end{restatable}
\begin{proof}[Proof sketch] We show this via induction on the number of online nodes. When $G$ has no online nodes, $\cV(G) = \cV(G(\pi)) = 0$.
The inductive step uses the one-step DP representations of $\cV(G) = \cV_G(L, 1)$ and $\cV(G(\pi)) = \cV_{G(\pi)}(L, 1)$. In particular, if $G$ has $t$ online nodes, then $\cV_G(L, 2)$ and $\cV_G(L \setminus \{u\}, 2)$ are full VTG computations on subgraphs of $G$ with $t-1$ online nodes. Applying the inductive hypothesis gives that $\cV_G(L, 2) \geq \cV_{G(\pi)}(L, 2)$ and $\cV_G(L, 2) \geq \cV_{G(\pi)}(L \setminus \{u\}, 2)$. Along with the fact that $\mathcal{N}_{G(\pi)}(1) \subseteq \mathcal{N}_G(1)$ as $G(\pi)$ is formed from $G$ by removing edges, these bounds imply $\cV(G) \geq \cV(G(\pi))$.
\end{proof}

To establish a lower approximation bound, given $G \in \textup{supp}(G(m,n,\cD,\Delta))$, we show that because each edge in $G$ exists in $G(\pi)$ for a $1 - \epsilon$ fraction of random $\lceil \frac{\epsilon}{2d\Delta}\rceil$-partitions $\pi$, $\cV(G(\pi))$ is at least $(1-\epsilon)\cV(G)$ in expectation.
\begin{restatable}{lemma}{VtgLB}\label{lemma: vtg-lower-bound}
    For any $G \in \textup{supp}(G(m, n, \cD, \Delta))$, $\epsilon > 0$, and $k = \lceil\frac{\epsilon}{2d\Delta}\rceil,$
    $\E_{\pi \sim \Pi_k}\left[ \cV(G(\pi)) \right] \geq (1-\epsilon)\cV(G).$
\end{restatable}
\begin{proof}[Proof sketch] First we show that $\cV(G)$ decomposes across edges. Recall that $\cV(G)$ is the expected value of $\onOPT(G)$. Given an arrival sequence $\vec{a} \in \{0,1\}^m$, $\onOPT$ outputs a deterministic matching $M(\vec{a})$. Therefore,
\[\cV(G) = \sum_{\vec{a} \in \{0,1\}^m} \left(\Pr[\vec{a} ]\sum_{e \in M(\vec{a} )} w_e\right).\]
After rearranging, we find $\cV(G) = \sum_{e \in E} \alpha_e w_e$ for
\[\alpha_e =\sum_{\vec{a}  \in \{0,1\}^m \; : \; e \in M(\vec{a} )} \Pr[\vec{a} ].\]
Crucially, for any hypercube partition $\pi$,
\begin{equation} \label{eq: edge-decomp-bound}
    \cV(G(\pi)) \geq \sum_{e \in E(G)} \alpha_e w_e \cdot 1\{e \in E(G(\pi))\}.
\end{equation}
The right-hand side is the expected value of the matching returned by an online algorithm on $G(\pi)$ which, for any $\vec{a}$, outputs $M(\vec{a}) \cap E(G(\pi))$. The left-hand side is the expected value of \onOPT$(G(\pi))$. The lemma statement then follows from Equation (\ref{eq: edge-decomp-bound}) and \cref{theorem: local-decomp}.
\end{proof}

Lemmas~\ref{lemma: vtg-upper-bound} and \ref{lemma: vtg-lower-bound} show that the function $\tilde{\cV}(G) = \E_{\pi \sim \Pi_k}\left[ \cV(G(\pi))\right]$ can approximate $\cV(G)$ with high accuracy over any b-RGG, but $\tilde{\cV}(\cdot)$ may not be $r$-local for any reasonable value of $r$. The main obstacle to overcome is that this would require the connected components of $G(\pi)$ to be of size at most $r$ under all partitions $\pi \in \textup{supp}(\Pi_k)$.
\paragraph{Local approximation via Monte Carlo estimation.}
The key observation is that unlike $\tilde{\cV}(\cdot)$, a random function that samples $\ell$ partitions $\pi_1, \dots, \pi_\ell \sim \Pi_k$ and outputs the sample estimate $\frac{1}{\ell}\sum_{i=1}^\ell \cV(G(\pi_i))$ achieves $r$-locality by only requiring that the connected components of $G(\pi_1), \dots, G(\pi_\ell)$ are of size at most $r$. We show that it is possible to choose $\ell$ such that with high probability over the joint draw of $G$ and the draw of $\ell$ partitions, the sample estimate is sufficiently accurate and still a local function.

\begin{restatable}{lemma}{localApprox} \label{lemma: local-approx}
     Let $\cD$ be a $\beta$-smooth distribution and $\Delta = O(N^{-1/d})$ where $N$ is sufficiently large. For all $\epsilon \in (0, 1/2]$, let $k=\lceil \frac{\epsilon}{2d\Delta} \rceil$. With probability $1-\delta$ over the draw of $G \sim G(m,n,\cD, \Delta)$ and the draw of $\ell\geq \frac{2}{\epsilon^2}\log\frac{4}{\delta}$ partitions $\pi_1, \dots, \pi_\ell\sim \Pi_k$, each $\cV(G(\pi_i))$ can be computed by a $O(\beta \log N)$-local function and
     \[\frac{1}{\ell} \sum_{i=1}^\ell \cV(G(\pi_i)) \geq (1 - \epsilon) \underset{\pi \sim \Pi_k}{\E}[\cV(G(\pi))].\]
\end{restatable}
\begin{proof}[Proof sketch]
To simplify notation, let $\vec{\pi} = \{\pi_1, \dots, \pi_\ell\}$, $\mathcal{G} = G(m,n,\cD, \Delta)$, $S_\ell(G, \vec{\pi}) = \frac{1}{\ell}\sum_{i=1}^\ell \cV(G(\pi_i))$, and $\bar{S}(G) = \E_{\pi}[\cV(G(\pi))]$. For $G \in \textup{supp}(\cG)$ and $\vec{\pi} \in \textup{supp}(\Pi_k)$,
let $A(G, \vec{\pi})$ be the event that $S_\ell(G, \vec{\pi}) \geq (1 - \epsilon) \bar{S}(G)$. Moreover, let $B(G, \vec{\pi})$ be the event the connected components of $G(\pi_i)$ are of size $O(\beta \log N)$ for each partition $\pi_i \in \vec{\pi}$. When this is the case, $\cV(G(\pi_i))$ can be computed exactly by a $O(\beta \log N)$-local function.

    By the tower property, we can rewrite
    \begin{align*}
       \underset{G \sim \cG, \vec{\pi} \sim \Pi_k}{\Pr}\left[A(G, \vec{\pi})^\complement\right] &= \underset{G \sim \cG}{\E}\left[ \Pr_{\vec{\pi} \sim \Pi_k}[A(G, \vec{\pi})^\complement \; | \; G]\right]\\
       \underset{G \sim \cG, \vec{\pi} \sim \Pi_k}{\Pr}\left[B(G, \vec{\pi})^\complement\right] &= \underset{\vec{\pi} \sim \Pi_k}{\E}\left[ \Pr_{G \sim \cG}[B(G, \vec{\pi})^\complement \; | \; \vec{\pi}]\right].
    \end{align*}
    
\begin{figure*}[!htb]
\centering
    \includegraphics[width=\textwidth]{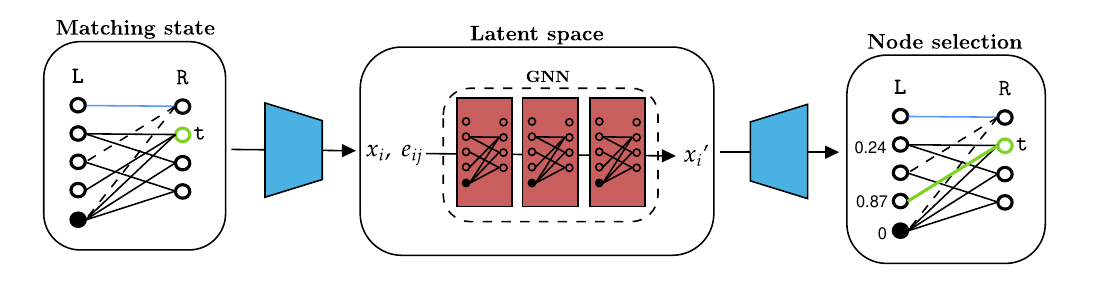}
    \caption{\MG's GNN-based matching subroutine.}
    \label{figure:gnn_algo}
\vskip 0in
\end{figure*}

 By \cref{theorem: local-decomp} and a union bound over the $\ell$ partitions, 
    \begin{equation}\underset{G \sim \cG}{\Pr}\left[B(G, \vec{\pi})^\complement \; | \; \vec{\pi}\right] \leq O\left(\frac{\ell}{N}\right) \leq \frac{\delta}{2}\label{eq:bound_B}\end{equation}    for $N$ sufficiently large.

    Next, for fixed $G \in \textup{supp}(\cG)$ and $\vec{\pi} \sim \Pi_k$, the $\cV(G(\pi_i))$'s are i.i.d. random variables which take values in the interval $[0, \cV(G)]$ by \cref{lemma: vtg-upper-bound}. A Hoeffding bound then implies that $\Pr_{\vec{\pi}\sim \Pi_k}[A(G, \vec{\pi})^\complement \; | \; G] \leq \delta/2$. This bound and Equation~\eqref{eq:bound_B} imply the lemma statement.
\end{proof}
     
Finally, we give our main theorem.
\begin{restatable}{theorem}{mainTheorem}
\label{theorem:vtg_local_approx}
     Given a $\beta$-smooth distribution $\cD$ over $[0,1]^d$ and $\Delta = O(N^{-1/d})$, for sufficiently large $N$, the VTG function $\cV$ is $\left(O(\beta \log N\right), \epsilon, \delta)$-locally approximable over $G(m, n, \cD, \Delta)$ for all $\epsilon \in (0, \frac{1}{2}]$ and $\delta \in (0,1].$
\end{restatable}
\begin{proof}[Proof sketch]
Let $k = \lceil\frac{\epsilon}{2d\Delta}\rceil$ and $\epsilon' = 1 - \sqrt{1 - \epsilon}$ so $(1 - \epsilon')^2 = 1- \epsilon$. By \cref{lemma: vtg-lower-bound} with $\epsilon = \epsilon'$, 
$\E_{\pi \sim \Pi_k}[\cV(G(\pi))] \geq (1 -\epsilon') \cV(G)$
 for all $G \in \textup{supp}(G(m,n,\cD, \Delta))$. By \cref{lemma: local-approx}, $\E[\cV(G(\pi))]$ can be approximated to error $1 -\epsilon'$ via a local Monte Carlo estimate. This implies the theorem statement.
\end{proof}

\section{Experiments}
In this section, we introduce \MG\footnote{\url{https://github.com/anders-wikum/GNN-OBM}} and demonstrate its performance against state-of-the-art baselines across a broad range of graph families and problem regimes. We describe our model and experimental setup in \cref{section:model,section:experimental setup}, then present results from experiments in \cref{section:results}.
More implementation details can be found in \cref{appendix:experimental_details}.

\subsection{Model}
\label{section:model}

\paragraph{Learned matching model.}
Recall that for an \RH input $G$, arriving online node $t$, and set of offline nodes $S$, the online optimal algorithm \onOPT skips $t$ if  $\cV_G(S, t + 1) > \max\{\cV_G(S, t, u) : u \in \mathcal{N}_G(t)\cap S\}$ and otherwise matches $t$ to the neighbor that maximizes $\cV_G(S, t, u)$. Fundamentally, \onOPT is a greedy algorithm with respect to VTG. \MG replaces the VTG computations in \onOPT with approximations from a GNN. See \cref{figure:gnn_algo} for an overview. To start, the current \emph{matching state} is encoded as an attributed graph. A matching state includes the input graph $G$, arriving online node $t$, and set of available offline nodes $S$. This attributed graph is fed into a GNN, which outputs an approximate VTG associated with each feasible action. Finally, the decision with the highest predicted VTG is chosen. This process is repeated on all matching states encountered while running on $G$ until no online nodes remain.

\paragraph{Training protocol.}
The GNN underlying \MG is trained to approximate VTG via supervised learning. Since computing targets involves solving an exponential-sized DP, we construct a training set of 2000 instances on 16 nodes. We generate matching states and targets from each instance via teacher forcing by following the decisions of \onOPT.
\begin{figure*}[t]
\vskip 0.15in
    \centering
    \includegraphics[width=0.95\textwidth]{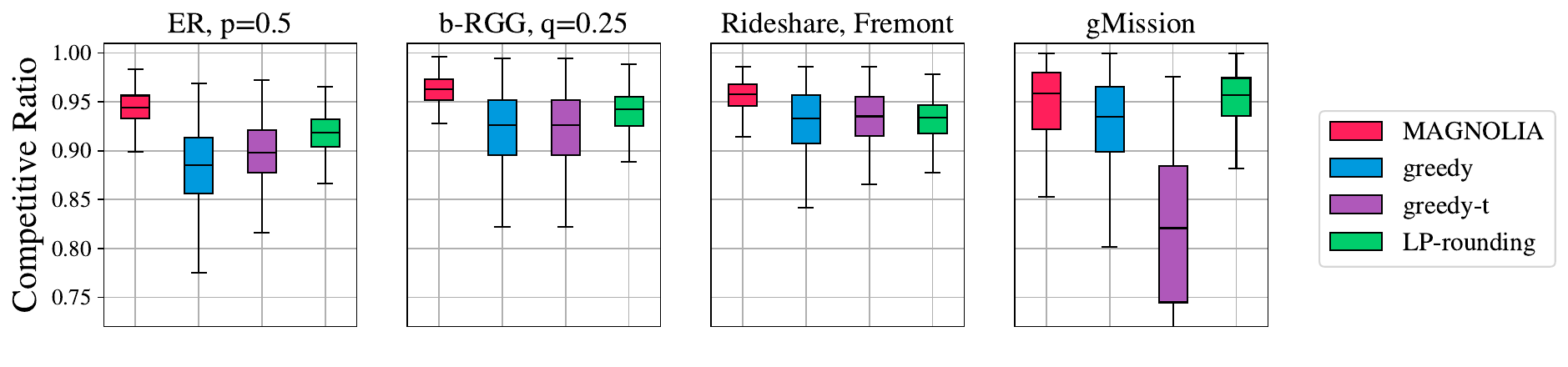}
    \caption{Boxplot showing the distribution of competitive ratios for \MG and baselines across graph configurations. All graphs are of size (10\x20), and results for additional configurations are available in \cref{appendix: base_more_params}. }
    \label{figure:baseline}
\vskip -0.2in
\end{figure*}

\begin{figure*}[t]
\vskip 0.2in
\centering
    \includegraphics[width=0.95\textwidth]{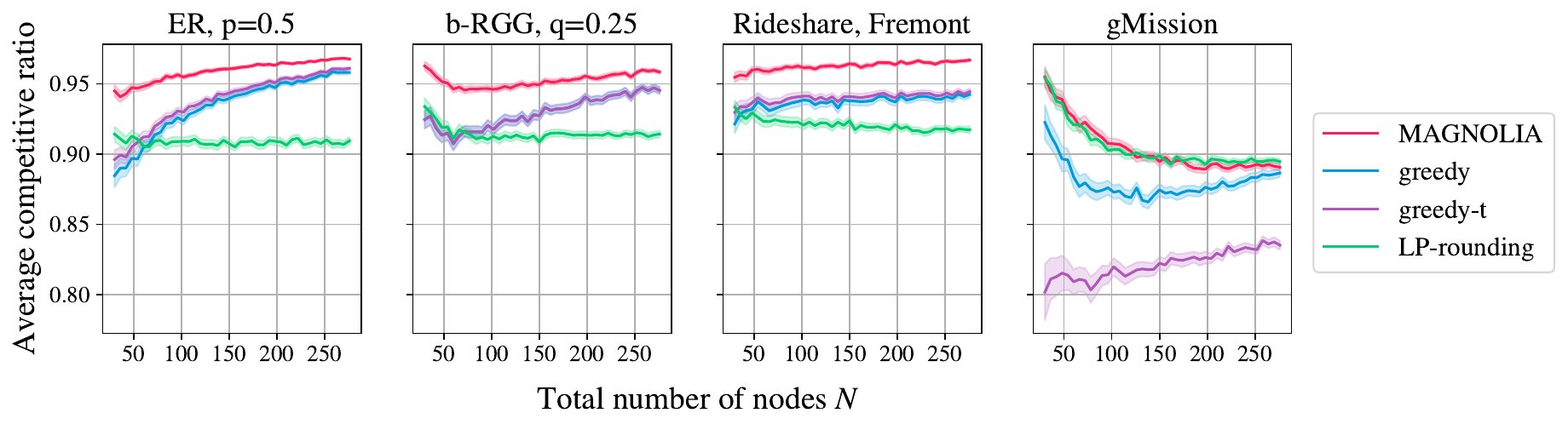}
    \caption{Evolution of competitive ratio over graphs of increasing size for a GNN trained on graphs of size (6\x10). All test graphs have a 2:1 ratio of online to offline nodes. Because the threshold $t$ for \greedyt is selected from a validation set over diverse graph configurations, it does not perform equally well on all configurations. Results for additional graph configurations are in \cref{appendix: size_more_params}}
    \label{figure:generalization_graph_size}
\vskip -0.1in
\end{figure*}
\subsection{Experimental setup}
\label{section:experimental setup}

\paragraph{Instance generation.}
Training and testing are done on synthetic and semi-synthetic inputs comprised of a weighted bipartite graph and arrival probabilities $p_t \sim U(0, 1)$. We generate synthetic bipartite graphs drawn from the Erd{\H o}s-R{\'e}nyi (ER) \cite{Erdos60:Evolution}, Barab{\'a}si-Albert (BA) \cite{Albert02:Statistical}, and geometric (b-RGG) random graph families. To simulate performance on real-world crowdsourcing tasks, we generate semi-synthetic graphs from OSMnx \cite{Boeing17:OSMnx}, a library that encodes road networks for use in ride-sharing applications, and the gMission dataset \cite{Chen14:gMission}, whose base graph comes from crowdsourcing data for assigning workers to tasks. Details on bipartite graph generation are in \cref{appendix: graph-generation}.

In terms of notation, we say a bipartite graph has shape $(|L| \text{\x} |R|)$ if it has $|L|$ offline nodes and $|R|$ online nodes. A graph \emph{configuration} refers to a graph family and fixed set of generating parameters (e.g., ER with $p=0.5$).

\paragraph{Evaluation.} 
Given an input graph $G$, a matching model $M$ outputs a sequence of matching decisions as the arrival $\vec{a}_t \in \{0,1\}$ of each online node $t$ is revealed. We evaluate the performance of $M$ on $G$ by taking the average \emph{competitive ratio} over $\ell$ realizations of the arrival vector $\vec{a}$:
\[\CR(M, G) = \frac{1}{\ell}\sum_{j=1}^\ell \frac{M(G, \vec{a^{(j)}})}{\OPT(G, \vec{a^{(j)}})}.\]
    Here $M(G, \vec{a})$ is the weight of the matching returned by $M$ on graph $G$ with arrival sequence $\vec{a}$, and $\OPT(G, \vec{a})$ is the weight of the max-weight matching in $G$ based on \emph{a priori} knowledge of $\vec{a}$. Similarly, we evaluate the performance of $M$ over a graph configuration $\cG$ by averaging the input-wise competitive ratio $\CR(M, G)$ over many input graphs $G \sim \cG$.
All competitive ratios are computed with respect to the offline optimal algorithm $\OPT$, since computing \onOPT is intractable for graphs of reasonable size.

\paragraph{Baselines.} 
We compare our learned models to several strong baselines. Upon the arrival of an online node, \greedy picks the maximum weight available edge. To better trade off between short and long-term rewards, \greedyt \cite{Alomrani22:Deep}, makes the same decision as \greedy if the maximum edge value is above some threshold $t$, and skips otherwise. This threshold parameter is tuned to maximize the average competitive ratio over a diverse validation set. Finally, \LP is an  \RH approximation algorithm by \citet{Braverman22:Max-weight}, which achieves an approximation ratio of $0.632$ and outperforms the current best-published approximation algorithm \citep{Naor23:Online} in practice.

To augment the potential value of \MG in practice, a central tenant in its design is that it should demonstrate strong out-of-distribution generalization to unseen graph configurations. For example, to obtain strong performance on ER graphs of a certain density, \MG should not need to be trained on ER graphs with that density. As such, because the approach by \citet{Alomrani22:Deep} learns a tailored matching policy per graph configuration, we do not consider it as a baseline. Moreover, it is difficult to make a fair comparison---as an RL-based approach, \citet{Alomrani22:Deep} requires more training data, larger training graph sizes, and more compute time than \MG to learn an effective policy. Thus, any comparison with a shared training set would be ill-suited to their method.

\subsection{Results}
\label{section:results}

\begin{figure*}[tb]
\vskip 0.15in
    \centering
    \includegraphics[width=0.95\textwidth]{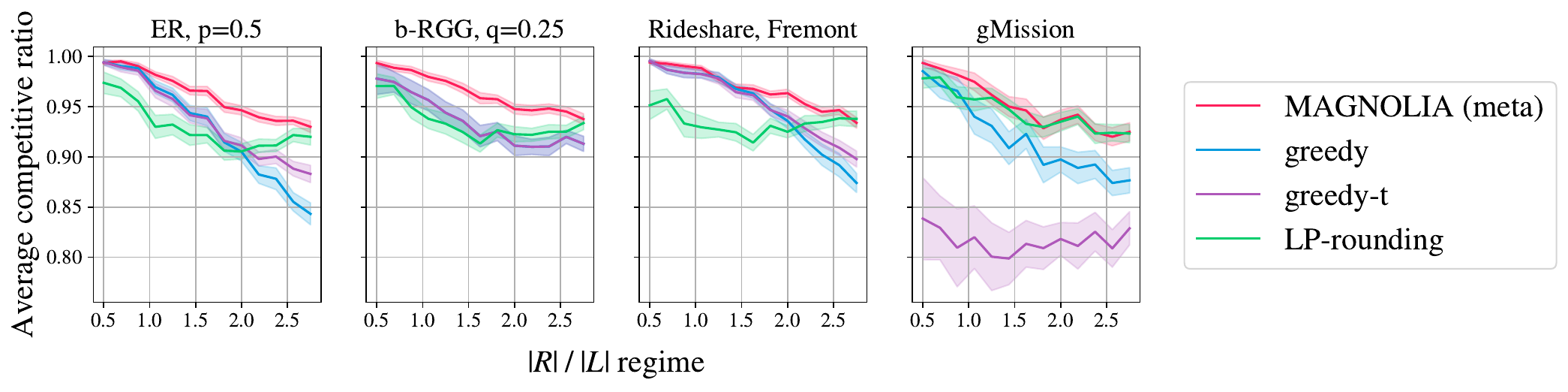}
    \caption{Evolution of competitive ratio over regimes for \MG enabled with a meta-GNN. For evaluation, $|L|$ is kept fixed at 16 offline nodes, and $|R|$ varies from 8 to 45 online nodes. Results for additional graph configurations are available in \cref{appendix: meta-models}.}
    \label{figure:meta_plots}
\vskip -0.15in
\end{figure*}

\paragraph{\MG makes good decisions.}
We train the GNN underlying \MG on a collection of 2000 \RH instances from 3 graph configurations, then evaluate performance on 6000 unseen instances from a broader selection of 12 configurations. \cref{figure:baseline} shows the distribution of competitive ratios for a representative sample of configurations. On these and others, \MG consistently achieves an average CR that is 2-5\% higher than the best baseline.

\paragraph{\MG shows size generalization.}
An important characteristic of GNNs trained for combinatorial tasks is the extent to which they generalize to graphs of larger size. Despite being trained exclusively on 16-node graphs, we see in \cref{figure:generalization_graph_size} that  \MG's performance largely remains consistent for inputs that are up to 18 times larger. For the single graph configuration where performance degrades as a function of graph size, the decline is in line with that of \texttt{LP-rounding}, whose approximation guarantees are invariant to graph size. 

\paragraph{Meta-models can improve regime generalization.} We observe that the ratio $|R|/|L|$ of online to offline nodes in an input materially impacts the performance of \RH algorithms. In light of this, we also study how \MG's performance generalizes across ratios $|R| / |L|$, or \emph{regimes}. Though \MG performs well on regimes similar to those its GNN was trained on, this does not translate to ones with different dynamics; a model trained in an ``offline-heavy" regime  ($|R| / |L| < 1$) tends to perform poorly in an ``online-heavy'' regime ($|R|/|L| > 1$), and~conversely. 

A natural solution is to replace \MG's GNN with a meta-model that selects between multiple GNNs---each trained on different regimes---on an instance-by-instance basis. In our experiments, we use a meta-GNN that selects between two GNNs trained on graph sizes (10\x6) and (6\x10) based on the predicted competitive ratio. \cref{figure:meta_plots} gives regime generalization results for \MG for this meta-model. We find that the meta-GNN recovers a simple configuration-dependent threshold rule on $|L|/|R|$. See the ablation tests in \cref{appendix: meta-models} for a comparison against a purely threshold-based meta-model. Our model consistently outperforms greedy baselines, and while \texttt{LP-rounding} is eventually better in very online-heavy regimes which are out-of-distribution for the GNN, \MG performs especially well in more balanced regimes where \texttt{LP-rounding} is worst.

\paragraph{\MG is robust to noisy inputs.}
One criticism of the $\RH$ model is that exact knowledge of both the underlying graph and arrival probabilities is impractical. Instead, it is often more reasonable to assume access to noisy estimates of these values coming from data or an ML model. A key question, then, is how robust these $\RH$ algorithms are to training and testing on noisy input. To test this, we run several experiments with varying levels of noise: for level $\rho$, $\mcN(0, \rho^2)$ noise is added independently to each edge weight $w_{ij}$ and each arrival probability $p_t$. \cref{figure:noise_robustness} shows how the performance of $\MG$ and baselines degrades as a function of $\rho$--- we see that while robustness to noise is configuration-dependent, \MG is consistently the most robust of the models we consider. 

The extent to which performance degrades for a particular graph configuration is related to the variance of edge weights incident on each online node. Adding noise increases the frequency with which an algorithm makes sub-optimal decisions, and so if this variance is high, mistakes are especially costly. This helps to explain the sizeable decline in competitive ratio for ER and Rideshare graphs that we observe in this experiment: the variance of edge weights incident on an online node is an order-of-magnitude larger in Erd{\H o}s-R{\'e}nyi and Rideshare graphs than in b-RGG and gMission graphs.

\begin{figure}[ht]
\vskip -0.2in
\centerline{\includegraphics[width=0.95\columnwidth]{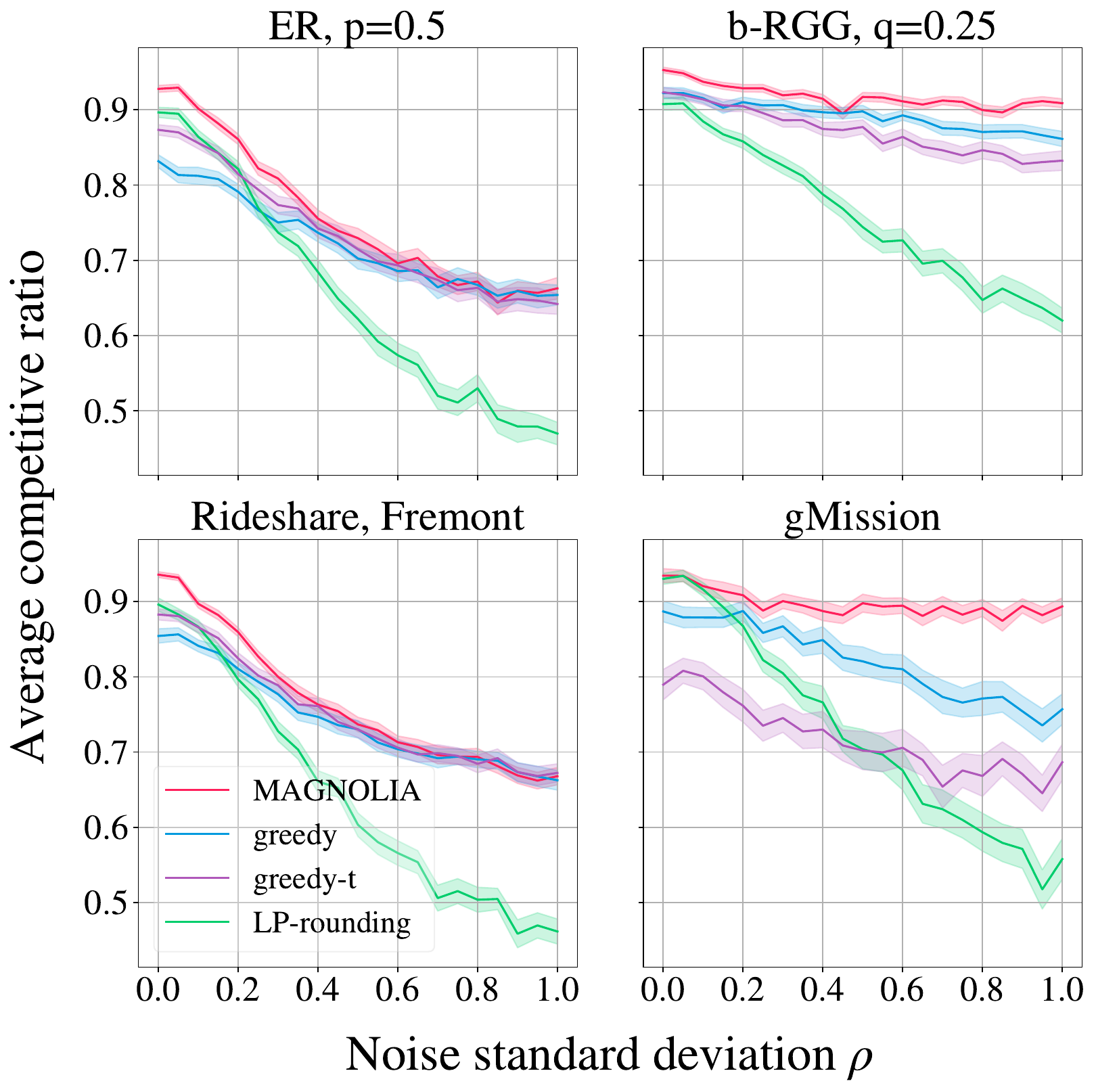}}
\caption{Evolution of competitive ratio as a function of noise level $\rho$ for graphs of size (10\x30). Results for additional graph configurations are available in \cref{appendix: noise-generalization}.}
\label{figure:noise_robustness}
\vskip -0.2in
\end{figure}

\section{Conclusions}

In this paper, we studied online matching in digital marketplaces, a problem of critical importance across various sectors such as advertising, crowdsourcing, ridesharing, and kidney exchange. We focused on the Online Bayesian Bipartite Matching problem, introducing a novel approach using GNNs to approximate the online optimal algorithm, which maximizes the expected weight of the final matching under uncertain future demand. We provided theoretical guarantees demonstrating that the value-to-go function can be efficiently approximated in bipartite random geometric graphs through local information aggregation---a process well-suited to GNNs. Empirically, our GNN-based algorithm, \MG, achieved strong performance against state-of-the-art baselines, showcasing strong generalization across different problem sizes and graph families.

This work exposes many directions for future research. For example, can we use the theoretical framework from \cref{sec:theory} to analyze the performance of GNNs on other problems beyond matching? In this vein, b-RGGs are one example of a graph family that exhibits locality and decomposibility. What other graph families exhibit these properties?

\section*{Impact statement}
This paper presents work whose goal is to advance the field of Machine Learning. There are many potential societal consequences of our work, none of which we feel must be specifically highlighted here.

\section*{Acknowledgements}
This work was supported in part by NSF grant CCF-2338226, AFSOR grant FA9550-23-1-0251, ONR grant N000142212771, and a National Defense Science \& Engineering Graduate (NDSEG) fellowship.

\bibliographystyle{plainnat}
\bibliography{references}

\newpage
\appendix
\onecolumn

\section{Proofs} \label{appendix: proofs}

\Decomposable*
\begin{proof}
    Consider $\pi \sim \Pi_k$, and let $\pi_\ell$ be the boundaries of $\pi$ along dimension $\ell$ for $\ell \in [d]$. For $i, j \in [N]$ such that $||\vec{x_i} - \vec{x_j}||_\infty \leq \Delta$, notice that $\vec{x_i}$ and $\vec{x_j}$ lie in different cells of $\pi$ precisely when in at least one dimension $\ell$, some point in $\pi_\ell$ falls in the interval $[(\vec{x_i})_\ell, (\vec{x_j})_\ell]$. By symmetry, the probability that this occurs is equal to the length of the interval $[(\vec{x_i})_\ell, (\vec{x_j})_\ell]$ over the measure of possible shifts $1/k$. Moreover, $\frac{|(\vec{x_i})_\ell - (\vec{x_j})_\ell|}{1/k} \leq k\Delta$, so
\begin{align*}
    \Pr_{\pi\sim\Pi_k}[\vec{x_i}\text{ and }\vec{x_j} \;\text{lie in different cells of $\pi$}] 
        &\leq 1 - \left (1 - k\Delta \right )^{d} \\
        &\leq 1 - \left (1 - \left(\frac{\epsilon}{2d\Delta} + 1\right)\Delta \right )^{d} \\
        &\leq 1 - \left (1 - \frac{\epsilon}{d} \right )^{d} \\
        &\leq 1 - (1 - \epsilon) \\
        &= \eps.
\end{align*}
\end{proof}

\ballsInBins*
\begin{proof}
    Let $c$ be a constant such that $K \geq cN$ for $N$ sufficiently large. Following Lemma 5.1 from \cite{Mitzenmacher05:Probability}, the probability that at least $M$ balls are dropped in bin 1 is at most
    
    \[\binom{N}{M} \left(\frac{\beta}{K}\right)^M\]
    
    by a union bound over the probability of each subset of $M$ balls being dropped in bin 1. In particular, there are $\binom{N}{M}$ possible subsets of balls, and the probability of selecting bin 1 for each of $M$ chosen balls is bounded above by $\left(\beta/
    K\right)^M$. By another union bound over the $K$ bins, the probability that any bin has a load of at least $M$ balls is at most
    \[K\binom{N}{M} \left(\frac{\beta}{K}\right)^M.\]
    
    We use the fact that $f(x) = x\binom{N}{M} \left(\frac{\beta}{x}\right)^M$ is decreasing for $x >0$ and the inequalities
    
    \[\binom{N}{M}\left(\frac{1}{N} \right)^M \leq \frac{1}{M!} \leq \left(\frac{e}{M}\right)^M\]

    to conclude that
    
    \[K\binom{N}{M}\left(\frac{\beta}{K} \right)^M \leq cN\binom{N}{M}\left(\frac{\beta}{cN} \right)^M = cN\left(\frac{\beta}{c}\right)^M \cdot \binom{N}{M}\left(\frac{1}{N} \right)^M \leq cN\left(\frac{\beta e}{cM}\right)^M.\]
    
    For $M \geq 3\beta \ln N / \ln\ln N$, the probability that any bin receives more than $M$ balls in bounded above by
    \begin{align*}
        cN\left(\frac{\beta e}{cM}\right)^M 
            &\leq cN\left(\frac{e\ln\ln N}{3c \ln N}\right)^{3\beta \ln N / \ln\ln N} \\
            &\leq cN\left(\frac{\ln\ln N}{c \ln N}\right)^{3\beta \ln N / \ln\ln N} \\
            &= e^{\ln c + \ln N}\left(e^{\ln\ln\ln N- \ln\ln N - \ln c}\right)^{3\beta \ln N / \ln\ln N} \\
            &= e^{\ln c + (1 - 3\beta) \ln N + 3\beta \ln N \left(\frac{\ln\ln\ln N - \ln c}{\ln\ln N}\right)}  \\
            &\leq ce^{-2 \ln N + 3\beta \ln N \left(\frac{\ln\ln\ln N - \ln c}{\ln\ln N}\right)} \\
            &\leq \frac{c}{N} \\
            &= O\left(1/N\right)
    \end{align*}
    for $N$ sufficiently large.
\end{proof}

\MaxLoad*
\begin{proof}
There are $k^d  = \Omega(N)$ cells of $\pi$, each of volume $\frac{1}{k^d}$. Since $\cD$ is $\beta$-smooth, the total density of $\cD$ in each cell is at most $\frac{\beta}{k^d}$. Treating the sampling of vectors from $\cD$ as a balls-into-bins process where balls are the $N$ vectors, and bins are the $\Omega(N)$ cells of $\pi$, the result follows from \cref{lemma: balls-into-bins}. We note that $\frac{\log N}{\log \log N} = O(\log N).$
\end{proof}

\LocalDecomposition*
\begin{proof}
(1) The definition of a b-RGG ensures that for $G \in \textup{support}(G(m, n, \cD, \Delta))$ an edge $(i,j)$ can only exist if $||\vec{x_i} - \vec{x_j}||_\infty \leq \Delta$. By \Cref{lemma: cut_prob}, $\vec{x_i}$ and $\vec{x_j}$ belong to the same cell of $\pi \sim \Pi_k$ with probability at least $1- \epsilon$. Equivalently, $i$ and $j$ belong to the same subgraph of $G(\pi)$ with probability at least $1-\epsilon$.

(2) \Cref{cor: max-load} implies that for $\pi \in \textup{support}(\Pi_k)$, the maximum number of latent embeddings of $G \sim G(m, n, \cD, \Delta)$ in any cell of $\pi$ is $O(\beta \log N)$ with probability at least $1 - O(1/N)$. The result follows from the observation that nodes in the same subgraph of $G(\pi)$ must have latent embeddings in the same cell of $\pi$. 
\end{proof}

\VtgUB*
\begin{proof}
 We will use an inductive argument on the number of online nodes in $G$. It is clear that $\cV(G(\pi)) = \cV(G) = 0$ for any partition $\pi$ when $G$ has no online nodes since there is nothing to match. Now, assume that $\cV(G(\pi)) \leq \cV(G)$ for all graphs $G \in \textup{support}(G(m, n, \cD, \Delta))$ with at most $t-1$ online nodes for some $t \geq 1$ and for all partitions $\pi$. Let $G \in \textup{support}(G(m, n, \cD, \Delta))$ be a graph on $t$ online nodes and let $\pi$ be any hypercube partition. Then,
\begin{align*}
    \cV(G) 
        &= \cV_{G}(L, 1)\\
        &= (1 - \vec{p}_1) \cdot \cV_G(L, 2) + \vec{p}_1 \cdot \max\left\{\cV_G(L, 2), \max_{u \in \mathcal{N}_G(1)} \{ w_{1u} + \cV_G(L \setminus \{u\}, 2)\}\right\} \\
       &\geq (1 - \vec{p}_1) \cdot \cV_{G(\pi)}(L, 2) + \vec{p}_1 \cdot \max\left\{\cV_{G(\pi)}(L, 2), \max_{u \in \mathcal{N}_G(1)} \{ w_{1u} + \cV_{G(\pi)}(L \setminus \{u\}, 2)\}\right\} \\
      &\geq (1 - \vec{p}_1) \cdot \cV_{G(\pi)}(L, 2) + \vec{p}_1 \cdot \max\left\{\cV_{G(\pi)}(L, 2), \max_{u \in \mathcal{N}_{G(\pi)}(1)} \{ w_{1u} + \cV_{G(\pi)}(L \setminus \{u\}, 2)\}\right\} \\
      &= \cV_{G(\pi)}(L, 1) \\
      &= \cV(G(\pi)).
\end{align*}

The first inequality is an application of the inductive hypothesis, since $\cV_{G}(L, 2)$ and $\cV_G(L \setminus \{u\}, 2)$ are both full value-to-go computations on a subgraph of $G$ with $t-1$ nodes. The second follows from the fact that $\mathcal{N}_{G(\pi)}(1) \subseteq \mathcal{N}_G(1)$, as $G(\pi)$ is formed from $G$ by removing edges.
\end{proof}

\VtgLB*
\begin{proof}
\begin{proofwrap}
     It is helpful to first decompose the value-to-go $\cV(G)$ into a contribution from each edge $e \in E(G)$. To do so, we make use of the fact that $\cV(G)$ is the expected value of the matching returned by $\onOPT$. In greater detail, let $\vec{a}  \in \{0,1\}^m$ represent an arrival sequence of online nodes where node $t$ arrives if $\vec{a} _t = 1$ and does not arrive if $\vec{a} _t = 0$. The likelihood of observing different arrival sequences is governed by the arrival probability vector $\vec{p}$. Namely, for $\vec{a} \in \{0,1\}^m$,
\[\Pr[\vec{a}] = \prod_{t=1}^m \left(\vec{p}_t \cdot \vec{a}_t + (1-\vec{p}_t) \cdot (1-\vec{a}_t)\right).\]
Notice that all randomness in the output of \onOPT comes from the random arrivals, so given a fixed arrival sequence $\vec{a}$, \onOPT returns a deterministic matching $M(\vec{a})$. Then, we can write

\[\cV(G) = \sum_{\vec{a} \in \{0,1\}^m} \left(\Pr[\vec{a} ] \cdot \sum_{e \in M(\vec{a} )} w_e\right) = \sum_{e \in E(G)} w_e \cdot \left( \sum_{\vec{a}  \in \{0,1\}^m \; : \; e \in M(\vec{a} )} \Pr[\vec{a} ]\right) = \sum_{e \in E(G)} \alpha_e w_e,\]

where

\[\alpha_e =\sum_{\vec{a}  \in \{0,1\}^m \; : \; e \in M(\vec{a} )} \Pr[\vec{a} ].\]

Crucially, notice that for any partition $\pi$,
\[\cV(G(\pi)) \geq \sum_{e \in E(G)} \alpha_e w_e \cdot 1\{e \in E(G(\pi))\}.\]
The right-hand side is the expected value of the matching returned by an online algorithm on $G(\pi)$ which, for any arrival sequence $\vec{a}$, outputs $M(\vec{a}) \cap E(G(\pi))$. The left-hand side is the expected value of the matching returned by \onOPT on $G(\pi)$. It follows immediately from these facts and \cref{lemma: cut_prob} that
\[\underset{\pi \sim \Pi_k}{\E} \left[\cV(G(\pi))\right] \geq \underset{\pi \sim \Pi_k}{\E} \left[\sum_{e \in E(G)} \alpha_e w_e 
\cdot 1\{e \in E(G(\pi))\} \right] \geq (1 - \epsilon) \sum_{e \in E(G)} \alpha_e w_e  = (1-\epsilon) \cdot \cV(G).\]
\end{proofwrap}
\end{proof}

\localApprox*
\begin{proof}
    To simplify notation, we refer to the sample mean $\frac{1}{\ell}\sum_{i=1}^\ell \cV(G(\pi_i))$ and true mean $\underset{\pi \sim \Pi_k}{\E}[\cV(G(\pi))]$ as $S_\ell$ and $\E[S_\ell]$, respectively. Also let $\vec{\pi} = \{\pi_1, \dots, \pi_\ell\}$ be shorthand for an i.i.d. sample of $\ell$ partitions from $\Pi_k$, and let $\mathcal{G}$ be shorthand for $G(m,n,\cD, \Delta)$.

    Consider the following events. For $G \in \textup{support}(\cG)$ and $\vec{\pi} \in \textup{support}(\Pi_k)$,
    \begin{itemize}
        \item $A(G, \vec{\pi})$ is the event that the approximation $(1 - \epsilon) \cdot \E[S_\ell] \leq S_\ell \leq (1 + \epsilon) \cdot \E[S_\ell]$ holds on $G$ for partitions $\vec{\pi}$.
        \item $B(G, \vec{\pi})$ is the event the connected components of $G(\pi_i)$ are of size $O(\beta \log N)$ for each partition $\pi_i \in \vec{\pi}$. When this is the case, $\cV(G(\pi_i))$ can be computed exactly by a $O(\beta \log N)$-local function that simply computes VTG over a $O(\beta \log N)$-hop neighborhood.
    \end{itemize}

     We need to show that for $N$ sufficiently large, the event $A(G, \vec{\pi}) \land B(G, \vec{\pi})$ occurs with probability at least $1-\delta$ over the random draws of $G \sim \cG$ and $\vec{\pi} \sim \Pi_k$. Toward that end, notice that

    \begin{align*}
        \underset{G \sim \cG, \vec{\pi} \sim \Pi_k}{\Pr}\left[A(G, \vec{\pi}) \land B(G, \vec{\pi})\right]
            &= 1 - \underset{G \sim \cG, \vec{\pi} \sim \Pi_k}{\Pr}\left[A(G, \vec{\pi})^\complement \lor B(G, \vec{\pi})^\complement\right] \\
            &\geq 1 - \underset{G \sim \cG, \vec{\pi} \sim \Pi_k}{\Pr}\left[A(G, \vec{\pi})^\complement\right] -  \underset{G \sim \cG, \vec{\pi} \sim \Pi_k}{\Pr}\left[B(G, \vec{\pi})^\complement\right]. 
    \end{align*}
    We have from the tower property of conditional expectation that
    \begin{align*}
       \underset{G \sim \cG, \vec{\pi} \sim \Pi_k}{\Pr}\left[A(G, \vec{\pi})^\complement\right] = \underset{G \sim \cG}{\E}\left[ \Pr_{\vec{\pi} \sim \Pi_k}[A(G, \vec{\pi})^\complement \; | \; G]\right]
    \end{align*}
    and
    \begin{align*}
       \underset{G \sim \cG, \vec{\pi} \sim \Pi_k}{\Pr}\left[B(G, \vec{\pi})^\complement\right] = \underset{\vec{\pi} \sim \Pi_k}{\E}\left[ \Pr_{G \sim \cG}[B(G, \vec{\pi})^\complement \; | \; \vec{\pi}]\right].
    \end{align*}

 By \cref{theorem: local-decomp} and a union bound over the $\ell$ drawn partitions, we have that
    \[\underset{G \sim \cG}{\Pr}\left[B(G, \vec{\pi})^\complement \; | \; \vec{\pi}\right] \leq O(\ell/N) \leq \delta/2\]
    for $N$ sufficiently large.

    To bound $\underset{\vec{\pi}\sim \Pi_k}{\Pr}[A(G, \vec{\pi})^\complement \; | \; G]$, first notice that
    for fixed $G \in \textup{support}(\cG)$ and $\vec{\pi} \sim \Pi_k$, the $\cV(G(\pi_i))$'s are i.i.d. random variables which take values in the interval $[0, \cV(G)]$ by \cref{lemma: vtg-upper-bound}. Applying a standard Hoeffding bound, for $\ell \geq \frac{2}{\epsilon^2}\log(\frac{4}{\delta})$ sampled partitions the probability of a bad approximation is
    \begin{align*}
        \Pr\left[\left|S_\ell - \E[S_\ell]\right| \geq \epsilon \E[S_\ell]\right] 
            &\leq \Pr\left[\left|S_\ell - \E[S_\ell]\right| \geq \epsilon (1-\epsilon)\cV(G)]\right] &&\text{\cref{lemma: vtg-lower-bound}} \\
            &\leq 2\exp\left(-\frac{2\epsilon^2(1-\epsilon)^2 \cV(G)^2}{\ell \cdot \cV(G)^2/\ell^2}\right) \\
            &\leq 2\exp\left(-\ell\epsilon^2/2\right) \\
            &\leq \delta/2.
    \end{align*}

    Thus for sufficiently large $N$, we've shown that
    \begin{align*}
        \underset{G \sim \cG, \vec{\pi} \sim \Pi_k}{\Pr}\left[A(G, \vec{\pi}) \land B(G, \vec{\pi})\right] &\geq 1-\underset{G \sim \cG}{\E}\left[ \Pr_{\vec{\pi} \sim \Pi_k}[A(G, \vec{\pi})^\complement \; | \; G]\right]-\underset{\vec{\pi} \sim \Pi_k}{\E}\left[ \Pr_{G \sim \cG}[B(G, \vec{\pi})^\complement \; | \; \vec{\pi}]\right] \\
        &\geq 1-\underset{G \sim \cG}{\E}\left[ \delta/2\right]-\underset{\vec{\pi} \sim \Pi_k}{\E}\left[ \delta/2]\right] \\
        &= 1 - \delta.
    \end{align*}
\end{proof}

\mainTheorem*
\begin{proof}
Let $k = \lceil\frac{\epsilon}{2d\Delta}\rceil$ and let $\epsilon' = 1 - \sqrt{1 - \epsilon}$ so that $(1 - \epsilon')^2 = 1- \epsilon$. By \cref{lemma: vtg-lower-bound} with $\epsilon = \epsilon'$, we have that
$\E_{\pi \sim \Pi_k}[\cV(G(\pi))] \geq (1 -\epsilon') \cdot \cV(G)$
 for all $G \in \textup{supp}(G(m,n,\cD, \Delta)$. Now, consider the random function $h(G)$ which samples $\ell = \frac{2}{\epsilon'^2}\log(\frac{4}{\delta})$ partitions $\pi_1, \dots, \pi_\ell$ from $\Pi_k$ then outputs $\frac{1}{|I|} \sum_{i \in I} \cV(G(\pi_i))$, where $I \subseteq [\ell]$ is the set of indices for which $\cV(G(\pi_i))$ is $O(\beta \log N)$-local. For sufficiently large $N$, it follows from \cref{lemma: local-approx} for $\epsilon = \epsilon'$ that with probability $1 - \delta$ over the draw of $G$ from $G(m, n, \cD, \Delta)$ and the randomness of $h$, both
     \[\frac{1}{\ell} \sum_{i=1}^\ell \cV(G(\pi_i)) \geq (1 - \epsilon')  \underset{\pi \sim \Pi_k}{\E}[\cV(G(\pi))] \geq  (1 - \epsilon) \cdot \cV(G)\]
and 
$h(G) = \frac{1}{\ell} \sum_{i=1}^\ell \cV(G(\pi_i)).$
\end{proof}

\section{Experimental Details}
\subsection{Value-to-go computation}
We provide pseudo-code for computing value-to-go:
\begin{algorithm}[ht] 
\label{alg: vtg-dp}
\caption{$\cV(S, t$)}\label{alg: VTG}\label{alg: vtg-dp}
\begin{algorithmic}
    
\STATE \textbf{Input:} Unmatched offline node set $S$, timestep $t$, map $M$ for memoizing intermediate computation, probability vector $\vec{p}$

\IF{$|S| = 0$ or $t = m + 1$}
    \STATE \textbf{return} 0
\ENDIF
\bigbreak
\IF{$(S, \; t+1) \notin M$}
    \STATE $M[(S,\; t+1)] = \cV(S, t+1)$
     
    \FOR{$u \in \mcN(t) \cap S$}
        \IF{$(S \setminus \{u\},\; t+1) \notin M$}
            \STATE $M[(S \setminus \{u\},\; t+1)] = \cV(S \setminus \{u\}, t+1)$
        \ENDIF
    \ENDFOR

\ENDIF
\bigbreak
\STATE $v_{max} = \max_{u \in N(t) \cap S} M[(S \setminus \{u\},\; t+1)]$ \\
\STATE \textbf{return} $(1 - \vec{p}_t) \cdot M[(S,\; t+1)] + \vec{p}_t \cdot \max\{M[(S,\; t+1)], v_{max}\}$
\end{algorithmic}
\end{algorithm}
\label{appendix:experimental_details}
\subsection{Graph generation} \label{appendix: graph-generation}
\subsubsection{Random graph families}

\paragraph{Erd{\H o}s-R{\'e}nyi (ER) \cite{Erdos60:Evolution}.} Given parameters ($m$, $n$, $p$), we generate a bipartite graph $G$ on $m$ online and $n$ offline nodes where the edge between each (online, offline) node pair appears independently with probability $p$. Edge weights are sampled from the uniform distribution $U(0,1)$.

\paragraph{Barab{\'a}si-Albert (BA) \cite{Albert02:Statistical}.} We use a process similar to the one described in \cite{Borodin20:Experimental} to generate scale-free bipartite graphs. Given parameters ($m$, $n$, $b$), we generate a bipartite graph $G$ on $m$ online and $n$ offline nodes via a preferential attachment scheme:
\begin{enumerate}
    \item Start with all $n$ offline nodes.
    \item For each online node, attach it to $b$ offline nodes sampled without replacement, where the probability of selecting offline node $u$ is proportional to
    \[\Pr[u] = \frac{\textup{degree}(u)}{\sum_{u'} \textup{degree}(u')}.\]
\end{enumerate}
Similarly, to ER, edge weights are sampled from the uniform distribution $U(0,1)$.
\paragraph{Geometric (b-RGG).} Given parameters ($m$, $n$, $q$) with $q \in [0,1]$, we generate a bipartite graph $G$ on $m$ online and $n$ by doing the following:
\begin{enumerate}
    \item Assign each online and offline node $u$ to a uniform random position $p_u$ in $[0,1]^2$.
    \item Connect online node $v$ to offline node $w$ such that
    \[w_{vw} \propto -\|p_v - p_w\|_2.\]
    \item Only keep the $q$ fraction of edges with the largest weight.
\end{enumerate}

\subsubsection{Semi-synthetic and real-word graphs}

\paragraph{OSMnx rideshare (Rideshare).} We generate a semi-synthetic ridesharing dataset using the OSMnx library \cite{Boeing17:OSMnx}. This dataset generation process is very similar to the one for b-RGG. To make it closer to a real-world application, we replace distances between random points with the time to drive between intersections in a city. 

For a given city and parameters ($m$, $n$, $t$), we uniformly sample intersections from a street map layout to generate locations for $n$ drivers and $m$ riders. There is an edge between driver $i$ and rider $j$ if the drive time from $i$ to $j$ is below some threshold $t$ (in practice, $t$ is set to 15 minutes). Approximate drive times are computed using the OSMnx library. Finally, edge weights $w_{ij}$ are generated such that
\[w_{ij} \propto -(\text{drive time from $i$ to $j$}).\]

This dataset can be thought of as a simple ridesharing application in a city. Drivers are idling, waiting to be matched to riders who arrive online at known locations. The application's goal is to minimize the sum travel time between all driver-rider pairs or, equivalently, to maximize $\sum_{e \in M} w_e$ where $M$ is the online matching created by the algorithm. The threshold $t$ is set to avoid riders having to wait too long for a car.

In practice, we use cities of varying sizes, from several thousands of inhabitants (e.g. Piedmont, California) to several hundreds of thousands of inhabitants (e.g. Fremont, California).

\paragraph{gMission.} gMission is a spatial crowdsourcing dataset where offline workers are matched to tasks that arrive online. There is an edge between a worker $u$ and a task $v$ if the worker can perform that task. The associated weight $w_{uv}$ is the expected payoff the worker will get from that task, computed based on some distance metric between the task's and the worker's feature vectors. We note that this setting is very similar to the Random Geometric Graphs we prove results for. Inputs are random node-induced subgraphs of the gMission base graph, which is made available by \citet{Alomrani22:Deep}.

\subsection{Node, edge, and graph features}
We augment our graphs with several node-level and graph-level features that the GNN can leverage to improve its predictions.

\paragraph{Node features.} On a particular instance, the GNN underlying \MG makes a decision for each arrival of a new online node. As the current ``matching state" evolves over time, some node features remain unchanged while others are dynamic. Static node features include a positional encoder for the nodes, a one-hot encoding for the skip node, and a binary mask for the offline nodes. In this way, the GNN can (1) differentiate each node from all others, (2) recognize the skip node as being different from other offline nodes, and (3) discriminate online from offline nodes. Dynamic node features include a one-hot encoding for the node the GNN is currently matching, and an arrival probability vector that is updated to $1$ (respectively, $0$) for nodes that have already arrived (respectively, not arrived) in the run of the algorithm. 

\paragraph{Edge features.} The weight $w_{ij}$ of each edge $(i,j)$ is encoded as a 1-dimensional edge feature.

\paragraph{Graph features.} We use a single graph-level feature: the ratio of remaining unmatched online nodes to offline nodes. Intuitively, an algorithm for online bipartite matching should get more greedy as this ratio goes down since greedy decisions are unlikely to lead to later conflicts.

\subsection{Architecture}
The convolutional layers of our GNN follow a GENConv architecture~\cite{Li20:DeeperGCN} and its implementation in PyTorch Geometric~\cite{Fey19:Fast}. The embedding update rule for this architecture mirrors the functional form of the dynamic program representation of value-to-go:
$$h_v^{(k)} = \text{MLP}\left (h_v^{(k-1)} + \max_{u \in \mcN(v)} \left\{\text{ReLU}(h_u^{(k-1)} + w_{vu})\right\} \right).$$
We compare different GNN architectures for VTG approximation in \cref{appendix:architecture-comparison}.

\subsection{Training error and model accuracy}
We train our model using mean squared error. On each training sample, the model is given a graph instance and the current online node $t$. It then tries to predict the value-to-go of all nodes in the graph. The only valid actions on step $t$ are to either match $t$ to one of its neighbors or not to match $t$ which is represented by matching $t$ to the skip node. Hence, the model's prediction is masked to only consider the neighbors of $t$ (which include the skip node) and we compute the mean squared error between those predictions and the actual value-to-go values given by the online optimal algorithm.

Model accuracy is used for hyperparameter tuning and is a good metric for the empirical performance of the GNN when used as an online matching algorithm. It is simply computed as the percentage of times the GNN chooses the same action as the online optimal algorithm. Here, choosing the same action could either mean matching to the same offline node or skipping the online node.

\subsection{Hyperparameter tuning} \label{appendix: hyperparameters}
We perform hyperparameter tuning using a validation set of size 300. We perform around $1000$ trials, tuning the parameters as described in \cref{table:hyperparams}. Each trial is evaluated by its validation set accuracy. The hyperparameters are tuned with Bayesian search \cite{Snoek12:Practical} and pruning from the Optuna library \cite{Akiba19:Optuna} to stop unpromising runs early. Similarly to the training setup, the hyperparameter tuning is done on small graphs (10\x6) and (6\x10) even though the eventual testing may be on larger graphs. 
All the training was done on an NVIDIA GeForce GTX Titan X. 

\bgroup
\def\arraystretch{1.2}
\begin{table}[tb]
    \centering
     \caption{Hyperparameter ranges}
     \begin{tabular}{ |c|c|c|}
     \hline
     &Hyperparameter& Values\\
     \hline 
       \multirow{4}{*}{\rotatebox[origin=c]{90}{GNN}} &$\#$ of message passing layers   & $\{1,\dots,6\}$ \\
     &\# of MLP layers &   $\{1,\dots,5\}$ \\
     &Hidden dimension size & $\{2^i \,| \,i \in \{1, \dots, 6\}\}$ \\
     &Dropout   & [0, 0.5] \\
     \hline
    \multirow{3}{*}{\rotatebox[origin=c]{90}{Training}} &Batch size   & $\{2^i \,| \, i \in \{1, \dots, 6\}\}$ \\
    &Epochs &   $\{2^i \,| \, i \in \{1, \dots, 8\}\}$ \\
    & Learning rate & $[1e-5, 1e-10]$ \\
    \hline
    \end{tabular}
    \label{table:hyperparams}

\end{table}
\egroup

\subsection{\MG using different architectures}
\label{appendix:architecture-comparison}
One of \MG's strengths is that it is a modular pipeline that can accept any GNN architecture as a VTG approximator. In \Cref{figure:other_gnns}, we validate the choice of the GENConv architecture by including a comparison with various state-of-the-art GNN models \cite{Li20:DeeperGCN,Morris18:Weisfeiler,Brody21:Attentive}. Note that GENConv and DeeperGCN have the same underlying GNN but use different layers and aggregation functions. We observe that all models achieve similar competitive ratios, with GENConv and DeeperGCN performing slightly better.

\begin{figure*}[ht]

\centering
    \includegraphics[width=0.95\textwidth]{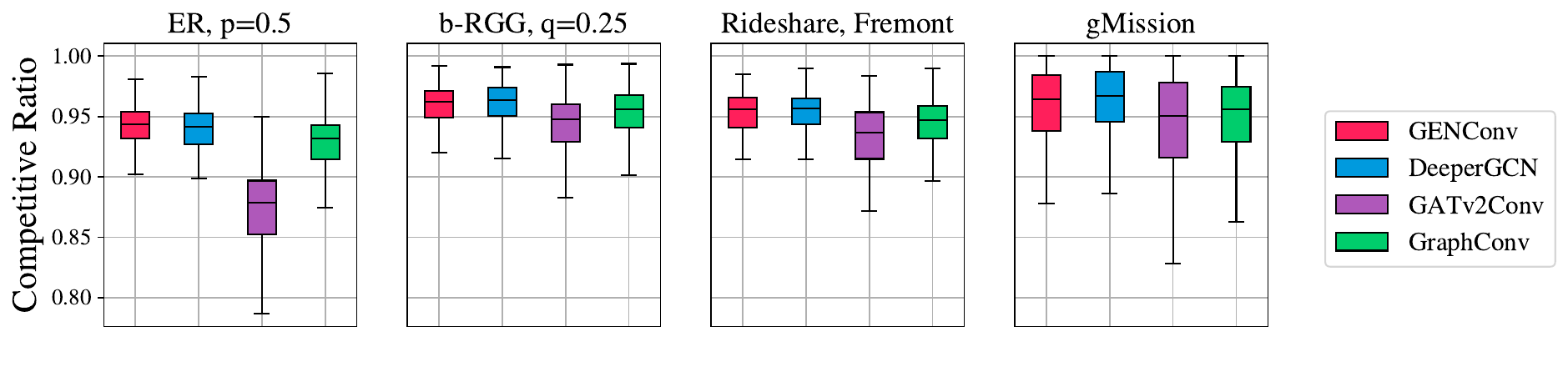}
    \caption{Boxplot showing the distribution of competitive ratios for \MG with different underlying GNN architectures across graph configurations. All graphs are of size (10\x20).}
    \label{figure:other_gnns}
\end{figure*}

\subsection{Complete results for \cref{section:results}}
For the results in \cref{appendix: base_more_params,appendix: size_more_params,appendix:reqs-for-size-gen}, the GNN underlying $\MG$ is trained on a set of 2000 instances of graphs from ER with $p=0.75$, BA with $b=4$, and GEOM with $q=0.25$ of size (6\x 10). Reported results are distributions and averages of competitive ratios from 6000 \RH unseen instances of size (10\x30) across 12 graph configurations. In particular, we elected to train on a subset of graph configurations since this considerably improves training time, and in our experience, leads to similar results.

The results in \cref{appendix: meta-models} come from a GNN-based meta model which is given as input two GNNs trained on graphs of size (6\x10) and (10\x6), respectively. The meta-GNN is trained on the competitive ratios achieved by each GNN on 2000 instances of graphs from ER with $p=0.75$, BA with $b=4$, and GEOM with $q=0.25$, each across graph sizes (10\x6), (8\x8), and (6\x10). Whereas the GNN-based model selects a GNN to run each instance on using predicted competitive ratios, the threshold-based meta algorithm simply runs on one GNN if the ratio of online nodes to offline nodes exceeds a fixed threshold $t$. Empirically, we found that $t=1.5$ performs well. Evaluation for both models once again happens over 6000 instances from the 12 graph configurations.

Finally, in \cref{appendix: noise-generalization}, we train \MG with a different GNN on each possible noise level $\rho$. The training and evaluation specifications for each of these noise-dependent GNNs are the same as those from \cref{appendix: base_more_params}.

\newpage
\subsubsection{\MG makes good decisions}\label{appendix: base_more_params}
\begin{table}[h]
    \centering
     \caption{Average competitive ratio by graph configuration with node ratio (10\x20).}
     \begin{tabular}{ |c|c|c|c|c|c|}
     \hline
     &Parameter & GNN& Greedy& Threshold Greedy & LP\\
     \hline 
       \multirow{3}{*}{ER} & $p=0.25$ & \textbf{0.945} & 0.881 & 0.887 & 0.929 \\
     & $p=0.5$ & \textbf{0.943} & 0.883 & 0.897 & 0.917 \\
     & $p=0.75$ & \textbf{0.949} & 0.905 & 0.914 & 0.915 \\
     \hline 
       \multirow{3}{*}{BA} & $b=4$ & \textbf{0.937} & 0.857 & 0.875 & 0.921 \\
     & $b=6$ & \textbf{0.944} & 0.885 & 0.896 & 0.916 \\
     & $b=8$ & \textbf{0.955} & 0.911 & 0.922 & 0.921 \\
     \hline 
       \multirow{3}{*}{GEOM} & $q=0.15$ & \textbf{0.978} & 0.938 & 0.938 & 0.958 \\

     & $q=0.25$ & \textbf{0.961} & 0.922 & 0.922 & 0.939 \\

     & $q=0.5$ & \textbf{0.950} & 0.924 & 0.924 & 0.921 \\

     \hline 
       \multirow{2}{*}{RIDESHARE} & $\text{city}=\text{Piedmont}$ & \textbf{0.957} & 0.935 & 0.939 & 0.936 \\

     & $\text{city}=\text{Fremont}$ & \textbf{0.957} & 0.929 & 0.933 & 0.930 \\

     \hline 
       \multirow{1}{*}{GMISSION} & - & \textbf{0.951} & 0.929 & 0.802 & \textbf{0.951} \\
     \hline 
    \end{tabular}
    \label{table:base_more_params}
\end{table}

\newpage

\subsubsection{\MG shows size generalization}\label{appendix: size_more_params}

\begin{figure*}[ht]
\centering
    \includegraphics[width=0.95\textwidth]{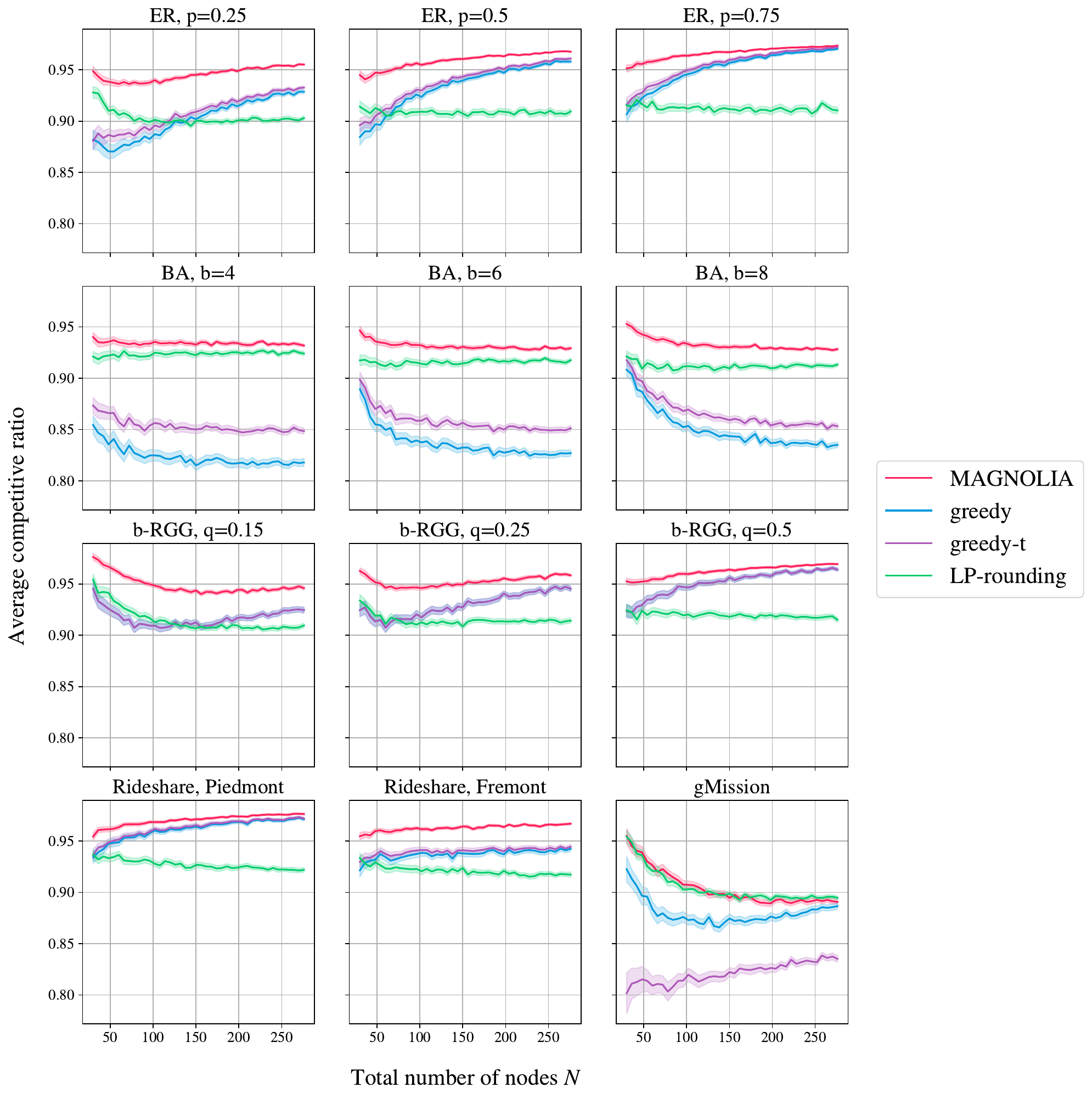}
    \caption{Evolution of competitive ratio over graphs of increasing size for a GNN trained on graphs of size (6\x10). All test graphs have a 2:1 ratio of online to offline nodes.}
    \label{figure:full_size_plots}
\end{figure*}

\newpage
\subsubsection{Requirements for size generalization}
\label{appendix:reqs-for-size-gen}
One of the advantages of our approach is that it performs well when trained on small graphs. Indeed, \cref{appendix: size_more_params} shows that \MG exhibits size generalization. Two questions remain:
\begin{enumerate}
    \item Would we observe better performance if \MG was trained on larger graphs?
    \item How small can the training graphs be while still exhibiting size generalization?
\end{enumerate}
To address these questions, we compare the size generalization of \MG when trained on graphs of varying size. We see in \Cref{figure:diff_gnn} that \MG shows strong generalization to graph size, even when trained on very small graphs. We observe that, surprisingly, the GNN trained on the smallest graphs (5\x3) performs the best on gMission. This can be explained by the fact that (5\x3) graphs are more likely to be sparse, making them similar to the very sparse gMission inputs.

\begin{figure*}[ht]

\centering
    \includegraphics[width=0.95\textwidth]{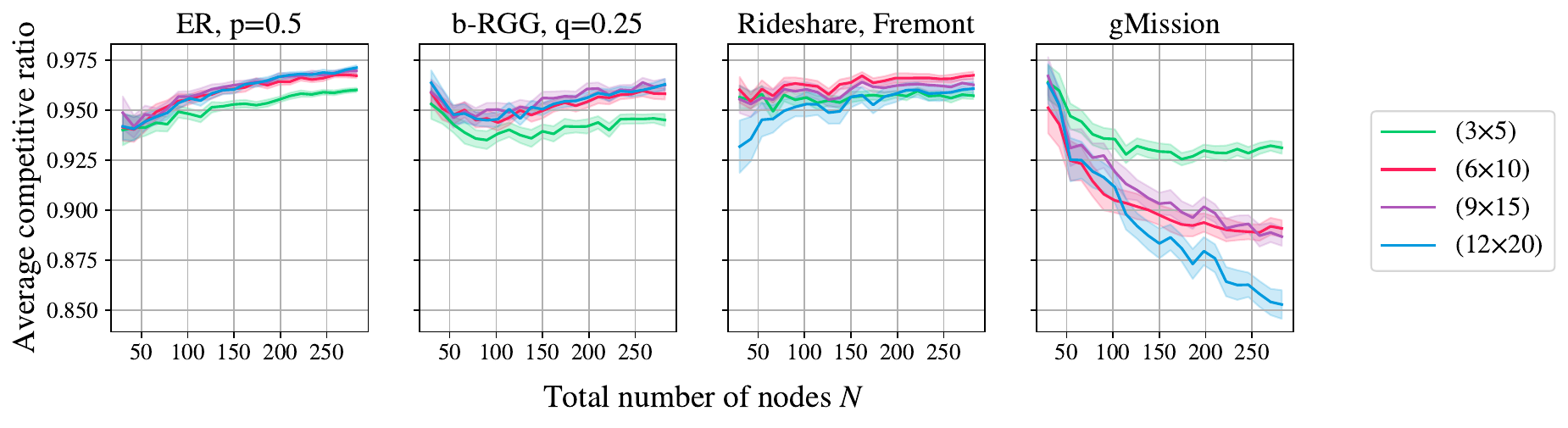}
    \caption{Evolution of competitive ratio over graphs of increasing size for GNNs trained on graphs of different sizes with the same (6:10) ratio. All test graphs have a 1:2 ratio of offline to online nodes.}
    \label{figure:diff_gnn}
\end{figure*}

\newpage
\subsubsection{Meta-model improves regime generalization}\label{appendix: meta-models}
\begin{figure*}[ht]
\centering
    \includegraphics[width=0.95\textwidth]{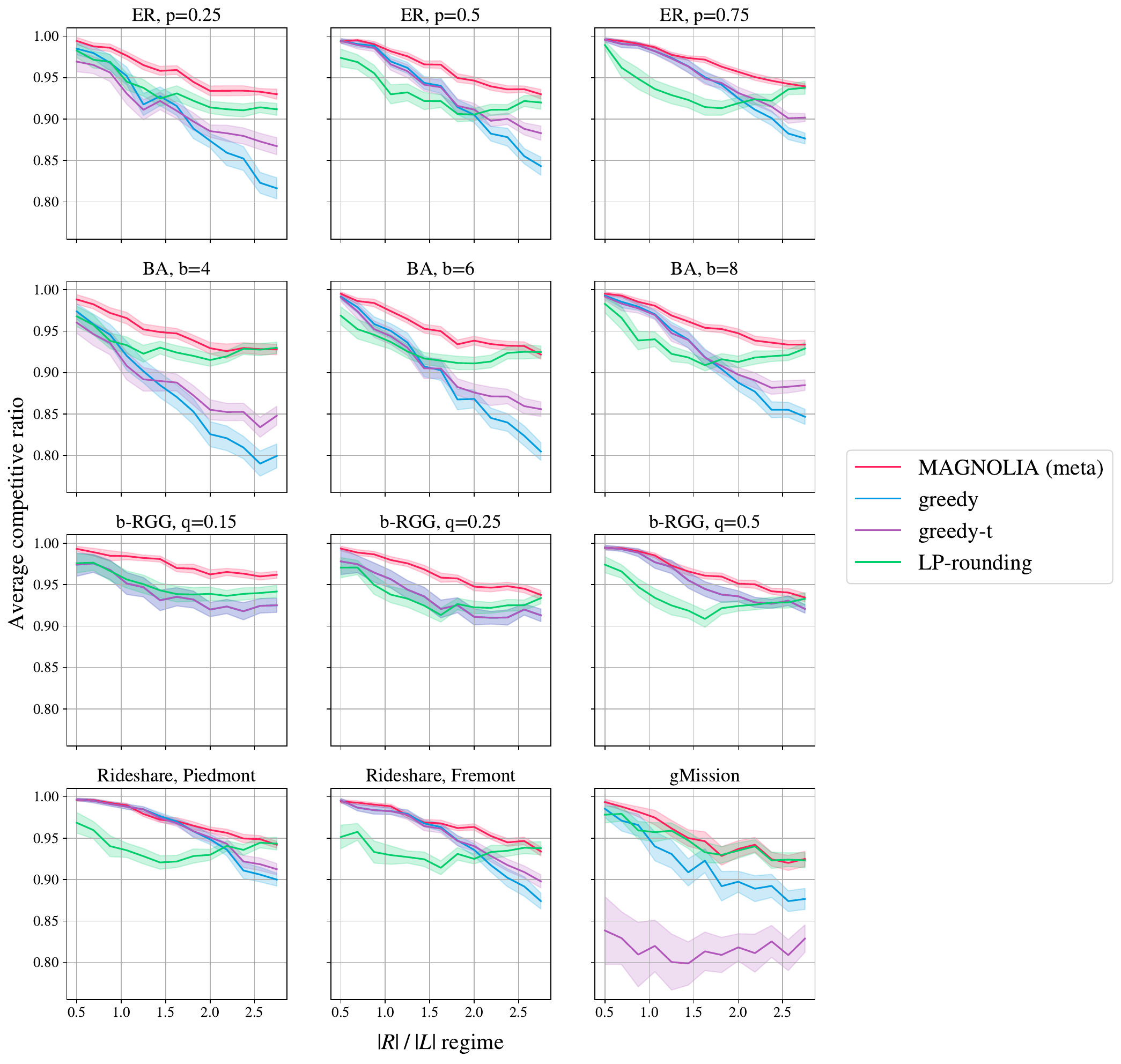}
    \caption{Evolution of competitive ratio over regimes for \MG enabled with a meta-GNN. For evaluation, $|L|$ is kept fixed at 16 offline nodes, and $|R|$ varies from 8 to 64 online nodes}
    \label{figure:full_meta_plots}
\end{figure*}
\newpage
\begin{figure*}[ht]
\centering
    \includegraphics[width=0.95\textwidth]{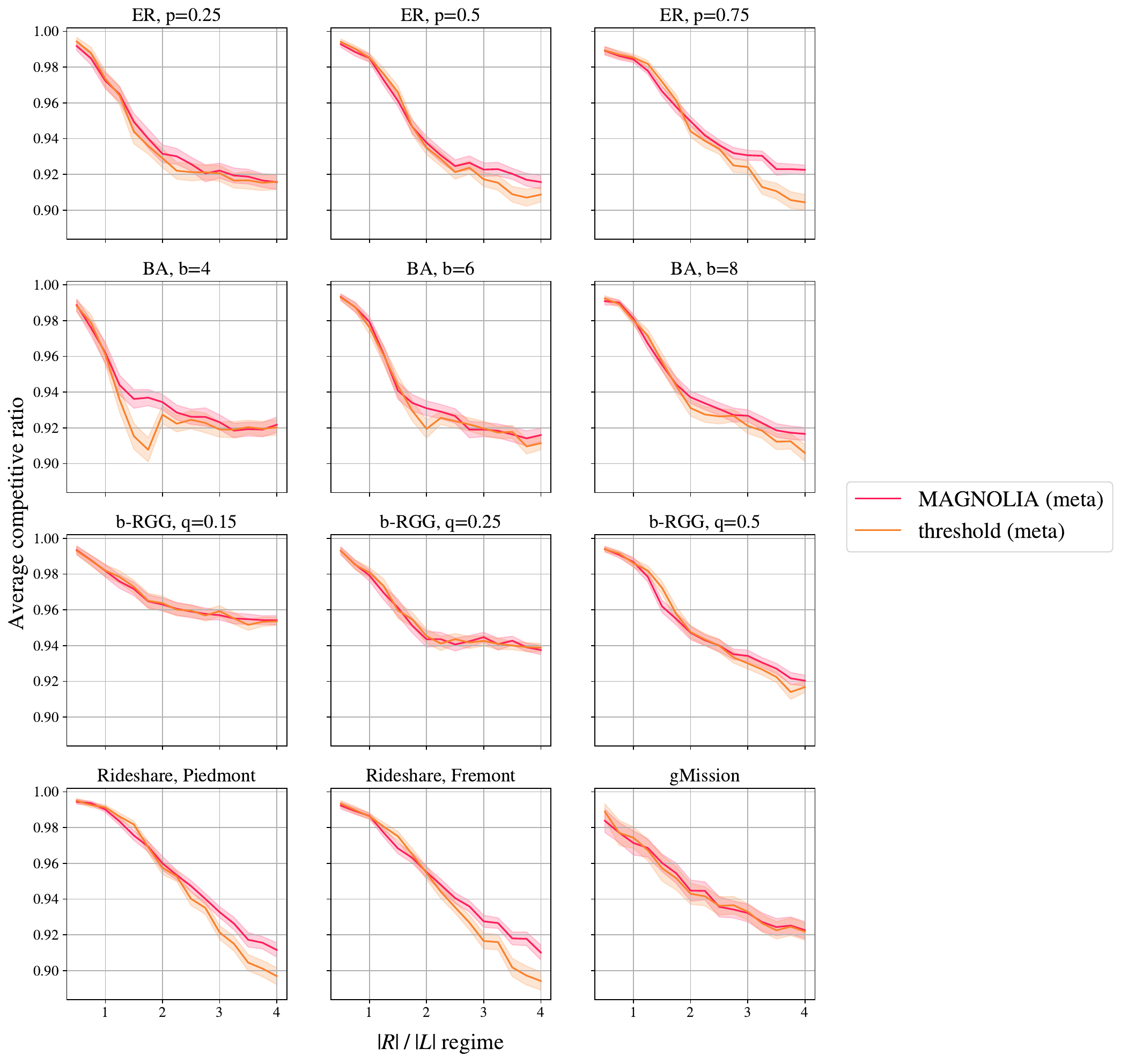}
    \caption{Evolution of competitive ratio over regimes for \MG enabled with a meta-GNN against simple threshold model. For evaluation, $|L|$ is kept fixed at 16 offline nodes, and $|R|$ varies from 8 to 64 online nodes}
    \label{figure:full_meta_plots}
\end{figure*}

\newpage
\subsubsection{\MG is robust to noisy inputs}
\label{appendix: noise-generalization}

\begin{figure*}[ht]
\centering
    \includegraphics[width=0.95\textwidth]{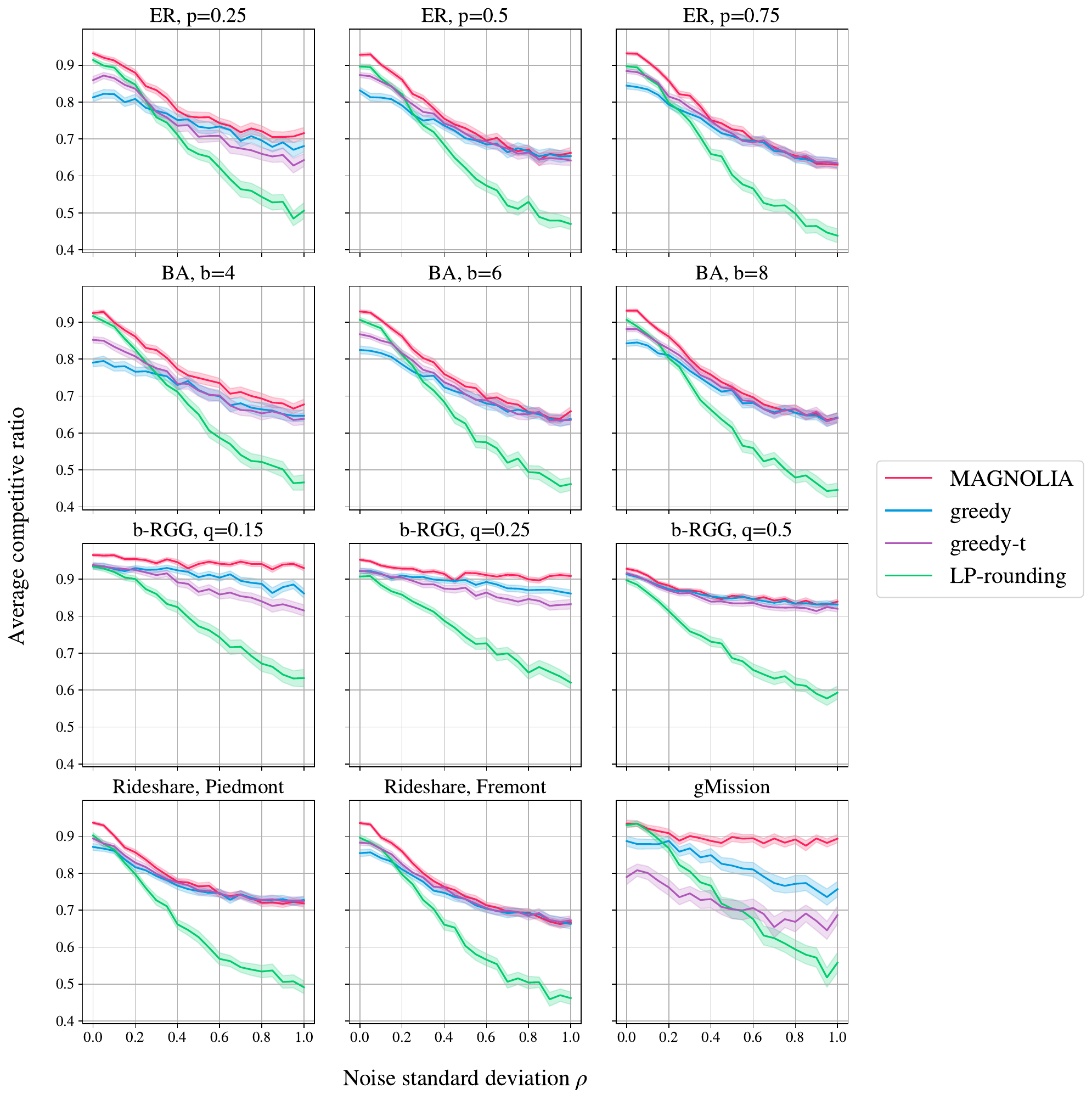}
    \caption{Evolution of competitive ratio as a function of noise level $\rho$ for graphs of size (10\x30). A $\mcN(0, \rho^2)$ noise is added independently to each edge weight and arrival probability.}
    \label{figure:noise_robustness_plots_all}
\end{figure*}

\end{document}